\def\extendedversion{1} 

\documentclass{article}
\usepackage{ijcai25withpagenumbers}

\usepackage{times}
\usepackage{soul}
\usepackage{url}
\usepackage[hidelinks]{hyperref}
\usepackage[utf8]{inputenc}
\usepackage[small]{caption}
\usepackage{graphicx}
\usepackage{amsmath}
\usepackage{amsthm}
\usepackage{booktabs}
\usepackage{algorithm} 
\usepackage{algorithmic}
\usepackage[switch]{lineno}

\usepackage{times}
\usepackage{soul}
\usepackage{url}
\usepackage[utf8]{inputenc}
\usepackage{graphicx}
\usepackage{amsmath}
\usepackage{amsthm}
\usepackage{amsfonts}   
\usepackage{amssymb}
\usepackage{booktabs}
\usepackage{algorithm}
\usepackage{algorithmic}
\usepackage[switch]{lineno}

\usepackage{hyperref}   
\usepackage{stmaryrd}
\usepackage{array}

\usepackage{multicol}

\usepackage{mathtools} 

\usepackage{nameref}
\usepackage{zref-xr,zref-user}
\zxrsetup{toltxlabel}
\zexternaldocument*{supplementary}

\newtheorem{theorem}{Theorem}
{Lemma}
{Proposition}
{Corollary}
\newtheorem{definition}
{Definition}
\newtheorem{example}
{Example}

\usepackage{xcolor}
\usepackage[color={1 1 0},type=typewriter]{pdfcomment}

\newcommand{\todo}[1]{{\noindent \color{blue} \textsc{TODO}: #1}}
\newcommand{\opt}[1]{{\noindent \color{cyan} \textsc{OPT}: #1}}
\newcommand{\noteG}[1]{\noindent\textcolor{magenta}{$\boldsymbol{\blacktriangleright}$#1}}
\newcommand{\tobo}[1]{{\noindent \color{blue} \textsc{TOBO}: #1}}

\usepackage{mathrsfs}
\def\sem{\mathscr{E}}
\def\ecem{\mathscr{C}}
\def\gau{\mathscr{U}}

\usepackage{tikz}
\usepackage{textcomp}
\usetikzlibrary{positioning,arrows,calc,fit,automata,shapes,backgrounds}

\newcommand{\s}{\scriptstyle}

\title{A Logic of General Attention Using Edge-Conditioned Event Models \\ (Extended Version)}

\author{
	Gaia Belardinelli
	\and
	Thomas Bolander \and
	Sebastian Watzl \\
	\affiliations
	Stanford University\\
	Technical University of Denmark\\
	University of Oslo\\
	\emails
	gaiabel@stanford.edu,
	tobo@dtu.dk,
	sebastian.watzl@ifikk.uio.no
}

\begin{document}
	
	\maketitle
	
	\begin{abstract}		
In this work, we present the first general logic of attention. Attention is a powerful cognitive ability that allows agents to focus on potentially complex information, such as logically structured 
propositions,
higher-order beliefs, or what other agents pay attention to.  This ability is a strength, as it helps to ignore what is irrelevant, but it can also introduce biases when some types of information or agents are systematically ignored. Existing dynamic epistemic logics for attention cannot model such complex attention scenarios, as they only model attention to atomic formulas. Additionally, such logics quickly become cumbersome, as their size grows exponentially in the number of agents and announced literals. Here, we introduce a logic that overcomes both limitations. First, we generalize edge-conditioned event models, which we show to be as expressive as standard event models yet exponentially more succinct (generalizing both standard event models and generalized arrow updates). Second, we extend attention to arbitrary formulas, allowing agents to also attend to other agents' beliefs or attention. Our work treats attention as a modality, like belief or awareness. We introduce attention principles that impose closure properties on that modality and that can be used in its axiomatization. Throughout, we illustrate our framework with examples of AI agents reasoning about human attentional biases, demonstrating how such agents can discover attentional biases.

	\end{abstract}
	
	\section{Introduction}
Attention is the crucial ability of the mind to select and prioritize specific subsets of available information 
\cite{watzl2017structuring}. Research in psychology suggests that agents who do not pay attention to something do not update their beliefs about it \cite{simons1999gorillas}. Therefore, 
while the ability to focus attention is a strength, as it helps agents ignore irrelevant information, restricting which information is learned can also introduce biases \cite{johnson2024varieties}: an agent with limited attention may never learn about certain news, some aspects of a candidate’s application, or what certain individuals (e.g.\ of a specific gender or nationality) have to say. Researchers have begun to investigate the significance of attention-driven biases in both humans and AI \cite{johnson2024varieties,munton2023prejudice}. AI systems that can robustly reason about attention could potentially detect such attentional biases and correct them. 

As a step toward that aim, we provide the first logic of general attention, here captured in a dynamic epistemic logic (DEL) framework \cite{baltag1998logic}. Attention is treated as a modality that restricts which parts of an event an agent learns. 
An agent who is not attending to a formula will mistake an event containing that formula for one that does not. Our work builds on earlier proposals that introduce DEL models of attention to capture its effects on agents’ beliefs. However, the first proposal only modeled a notion of all-or-nothing attention \cite{bolander2015announcements}, and later work extended it to apply to atomic propositions \cite{belardinelli2023attention}. 
While the latter can treat cases where a member of a hiring committee systematically attends only to, say, the research parts of applicants’ CVs, 
it cannot represent agents whose attention is biased against specific agents, 
e.g.\ paying attention to 
what candidate $a$ has to say about a research topic, but not candidate $b$~\cite{munton2023prejudice,smith2020epistemic}. Or consider the distinction between an agent who aims to learn by paying attention to the world and an agent who aims to learn by attending to what others pay attention to -- the latter arguably being an important component of 
social learning
\cite{rendell2010copy,boyd2011cultural}. 
To model such cases, we need a notion of attention that applies to arbitrary formulas.

A key obstacle to such a richer notion is that models for attention quickly become highly complex and their size grows at least exponentially large, even for all-or-nothing attention. A fully general theory requires technical innovations that allow to capture attention and its impact on beliefs in a tractable way.
To that goal, we adopt an edge-conditioned version of DEL \cite{bolander2018seeing} and generalize it so that every edge carries both a source and a target condition. Source conditions are used to encode an agent's current attentional state, while target conditions specify what the agent learns about her own attention. We show that this single modification allows for an exponentially more succinct reformulation of earlier event models for propositional attention \cite{belardinelli2023attention}. 
We also show that edge-conditioned event models are as expressive as standard event models, establishing exponential succinctness of edge-conditioned event models over standard event models. We further provide an axiomatization of edge-conditioned models and show that they are as expressive as generalized arrow updates \cite{kooi2011generalized}
and as succinct as them. Taken together, these results imply that edge-conditioned models provide a novel 
 event model formalism that generalizes and unifies the two common alternative formalisms in DEL: standard event models and generalized arrow updates. Due to these properties, it seems a good candidate for a new standard formalism in DEL.

Thanks to the increased simplicity and clarity of the new formalisation, we can now generalize the framework along another dimension, namely to accommodate attention to arbitrary formulas (not merely atomic ones). In this richer framework attention becomes a modality, in that way resembling belief or awareness (though, unlike the latter, attention is intrinisically dynamic as it affects what agents learn from events, and it is not used to restrict their standing beliefs). We argue that the behaviour of that modality can be governed by a family of attention principles, namely closure conditions that can be used to axiomatize a specific logic of attention.
The logic of general attention can model agents who attend to the world, to what other agents believe, or to what other agents attend to.  It can thus be used to study social attention and learning and the attentional biases mentioned above in potentially complex multi-agent settings. 

In summary, the paper makes the following contributions:

\vspace{-1mm} 
\begin{enumerate}\setlength{\itemsep}{0mm}
	\item \emph{We present the first general theory of attention}, allowing agents to attend to arbitrary formulas. 
		\item \emph{We provide the first event model formalism that generalizes and unifies 
		standard event models and generalized arrow updates}. We show that this new formalism is always at least as succinct as both, and can be exponentially more succinct. 
		\item \emph{We provide the first principles for general attention} that can be used to axiomatize specific attention notions. 
	\end{enumerate}

%


%
%

%
%
%
%
%


\vspace{-1mm} 
\if\extendedversion 1
The main text only contains proof sketches. Full proofs can be found in the appendix.  
\else
Due to page limits, we only provide proof sketches of results. \emph{All full proofs can be found in the extended version on Arxiv: \url{https://arxiv.org/pdf/2307.07448}}. 
\fi

	\section{DEL and Propositional Attention}\label{section:del and propositional attention}
	We are going to work with multiple distinct languages, in all of which 
	we use $Ag$ to denote a finite set of \emph{agents} and $P$ to denote a finite set of \emph{propositional atoms}.\footnote{So all our languages take the sets $Ag$ and  $P$ as parameters, but this dependency is kept implicit throughout the paper.} 
	The symbol $\mathcal{L}$ 
	%
	is used to denote any language extending the \emph{language of epistemic logic} $\mathcal{L}_\text{EL}$ given by the grammar: 
	\[
	\varphi::= \top \mid p \mid
	\neg\varphi\mid\varphi\wedge\varphi\mid B_a\varphi,
	\]  
	with $p \in P$, $a \in Ag$, and $B_a \varphi$ reads ``agent $a$ believes $\varphi$''. We define the other propositional connectives in the standard way. 
	Every language $\mathcal{L}$ has a set of \emph{atoms} $At(\mathcal{L})$ with $At(\mathcal{L}) \supseteq P$. 
	A \emph{literal} is an element of $At(\mathcal{L}) \cup \{\neg p \colon p \in At(\mathcal{L})\}$, and a \emph{propositional literal} is an element of $P \cup \{\neg p \colon p \in P\}$. For $p\in At(\mathcal{L})$, we denote a literal by $\ell(p)$, where either $\ell(p)=p$ or $\ell(p)=\neg p$.
	For any formula $\varphi$, $At(\varphi)$ denotes the set of atoms appearing in it. 
	
	\paragraph{DEL}
	We now introduce the standard ingredients of DEL 
	\cite{baltag1998logic,ditmarsch2007dynamic}.
	\begin{definition}[Kripke model]\label{def: kripke model} A \emph{Kripke model} 
		for $\mathcal{L}$ 
		is a tuple $\mathcal{M}=(W,R,V)$ where $W 
		\not=\emptyset$ 
		is a finite set of \emph{worlds}, $R:Ag\rightarrow \mathcal{P}(W^{2})$ assigns an \emph{accessibility relation} $R_a$ to each agent $a\in Ag$, and $V:W\rightarrow\mathcal{P}(At(\mathcal{L}))$ is a \emph{valuation function}. Where $w$ is the \emph{actual world}, we call $(\mathcal{M},w)$ a \emph{pointed Kripke model}.
	\end{definition}
	\begin{definition}[Standard event model]
		\label{def: standard event model} A \emph{standard event
			model} for $\mathcal{L}$ is a tuple $\mathcal{E}=(E,Q,pre)$
		where $E
		\neq\emptyset$ 
		is a finite set of \emph{events}, $Q:Ag\rightarrow \mathcal{P}(E^{2})$ assigns an \emph{accessibility relation} $Q_a$ to each agent $a\in Ag$, and $pre:E\rightarrow\mathcal{L}$ assigns a \emph{precondition} 
		to each event.
		The set of event models for $\mathcal{L}$ is denoted by $\sem(\mathcal{L})$.
		Where $E_d\subseteq E$ is a set of \emph{actual events}, $(\mathcal{E},E_d)$ is a \emph{multi-pointed standard event model}. 
		If $E_d = \{e \}$ for some
		\emph{actual event} $e \in E$, then $(\mathcal{E},E_d)$ is called a \emph{pointed standard event model}, also denoted $(\mathcal{E},e)$.\footnote{We often denote event models by $\mathcal{E}$ independently of whether we refer to an event model $(E,Q,pre)$, a pointed event model $((E,Q,pre),e)$ or a multi-pointed one, $((E,Q,pre),E_d)$. Their distinction will be clear from context.} 
	\end{definition}
	
	
	Intuitively, a (pointed) Kripke model represents an epistemic state, while a (pointed) event model represents an epistemic action or event happening. The \emph{product update operator}, defined next, expresses how an epistemic state is updated as the consequence of an epistemic event.
	\begin{definition}
		[Standard product update]
		\label{def: product update no post}  Let $\mathcal{M} = (W,R,V)$ be a Kripke model and $\mathcal{E} = (E,Q,pre)$ an event model, both 
		for the same language $\mathcal{L}$. 
		The \emph{product update} of $\mathcal{M}$ with $\mathcal{E}$ is $\mathcal{M} \otimes \mathcal{E} = (W',R',V')$ where:\footnote{The meaning of $(\mathcal{M},w) \vDash pre(e)$ depends on the semantics of $\mathcal{L}$. 
			Semantics of specific languages $\mathcal{L}$ are introduced later.} 
		
		\vspace{1pt}
		\noindent	
		\hspace{1pt}
		$W'=\{(w,e)\in W\times E\colon (\mathcal{M},w) \vDash pre(e)\}$,
	
	\vspace{1pt}
	\noindent	
	\hspace{1pt}
	$R'_a=\{((w{,}e),(v{,}f))\in (W')^2\colon (w{,}v)\in R_a,  (e{,}f)\in Q_a\}$,
	
	\vspace{1pt}
	\noindent	
	\hspace{1pt}
	$V'((w,e))=\{p\in At(\mathcal{L}) \colon p\in V(w)\}$.
	
	\noindent	
	Given a pointed Kripke model $(\mathcal{M},w)$ and a pointed or multi-pointed event model $(\mathcal{E}, E_d)$, we say that $(\mathcal{E},E_d)$ is \emph{applicable} in $(\mathcal{M},w)$ iff there exists a unique $e' \in E_d$ such that $(\mathcal{M}, w) \vDash pre(e')$. In that case, we define the \emph{product update} of $(\mathcal{M},w)$ with $(\mathcal{E},E_d)$ as the pointed Kripke model $(\mathcal{M},w) \otimes (\mathcal{E},E_d) = (\mathcal{M} \otimes \mathcal{E}, (w,e'))$.
\end{definition}
We define the language of \emph{DEL with standard event models} $\mathcal{L}_\text{DEL}$ as the language given by the grammar of $\mathcal{L}_\text{EL}$ extended with the clause $\varphi ::= [\mathcal{E}] \varphi$, where $\mathcal{E}$ is a pointed or multi-pointed standard event model. 
The formula 
$[\mathcal{E}] \varphi$ reads ``after $\mathcal{E}$ happens, $\varphi$ is the case''. Notice that $At(\mathcal{L}_\text{DEL}) = P$.
\begin{definition}[Satisfaction]\label{def: static truth} \label{def:satisfaction} 
	Let $(\mathcal{M},w )= ((W,R,V),w)$
	be a pointed Kripke model for $\mathcal{L}_\text{DEL}$. Satisfaction of $\mathcal{L}_\text{DEL}$-formulas in $(\mathcal{M},w )$  is given by the following clauses extended with the standard clauses for propositional logic:
	\noindent \begin{center}
		\begin{tabular}{lll}
			$(\mathcal{M},w) \vDash p$ & iff & $p\in V(w)$, \text{ where $p \in At(\mathcal{L}_\text{DEL})$}\tabularnewline
			$(\mathcal{M},w) \vDash B_a\varphi$ & iff & $(\mathcal{M},v) \vDash \varphi$ for all $(w,v)\in R_a$\tabularnewline
			$(\mathcal{M},w) \vDash [\mathcal{E}]\varphi$ & iff & 
			if $\mathcal{E}$ is applicable in $(\mathcal{M},w)$ then \\
			&&$(\mathcal{M},w) \otimes \mathcal{E} \vDash \varphi$.
		\end{tabular}
		\par\end{center}	
	We say that a formula $\varphi$ is \emph{valid} if $(\mathcal{M},w) \vDash \varphi$ for all pointed Kripke models $(\mathcal{M},w)$, and in that case we write $\vDash \varphi$.
\end{definition}

\paragraph{The language of propositional attention}
We add new atomic proposition $A_a p$ called an \emph{attention atom}, for all $a \in Ag$ and $p \in P$. The set of \emph{attention
	atoms} is $H=\{A_a p \colon p\in P, a\in Ag\}$. 
The \emph{language of propositional attention} $\mathcal{L}_\text{PA}$ is the language given by the grammar of $\mathcal{L}_\text{DEL}$ extended with the following clause: $\varphi ::= A_a p$, where $A_a p \in H$~\cite{belardinelli2023attention}. 
We read $A_a p$ as ``agent $a$ pays attention to whether $p$''.%
%
%
%
%
%
\footnote{\citeauthor{belardinelli2023attention}~[\citeyear{belardinelli2023attention}] used instead the notation $h_a p$ derived from the ``hearing'' atoms $h_a$ of Bolander \emph{et al.}~[\citeyear{bolander2015announcements}]. 
	The new notation $A_a p$ fits better with our generalizations in Section~\ref{sect:general}. 
}  
We have $At(\mathcal{L}_\text{PA}) = P \cup H$.
Satisfaction of formulas in $\mathcal{L}_\text{PA}$ is exactly as in Definition~\ref{def: static truth}, except we replace $\mathcal{L}_\text{DEL}$ by $\mathcal{L}_\text{PA}$ everywhere (the difference only amounts to the addition of the new set of atoms $H$).

\begin{example}\label{ex:static}
	Ann is about to review the CV of an applicant for a position in her lab. To ensure an unbiased hiring process, her lab has adopted an AI agent to detect the presence of attentional biases and assess whether some aspects of the candidates' CVs receive more attention than others. The AI agent does not have any information about which aspects Ann is prioritizing, and only knows that she has not read the CV yet and so has no information about it. However, it correctly assumes that Ann overestimates her own attentional capacities: Ann believes that she (and everybody else) will pay attention to every aspect of the CV. In reality, Ann focuses more on the candidate's research track record, and does not pay so much attention to other important factors such as contributions to diversity. The AI agent has already read and paid attention to all aspects of the CV. This situation is represented in Fig.~\ref{figure: propositional attention}. We have e.g. $(\mathcal{M},w) \vDash A_a p \land \neg A_a q$: Ann ($a$) pays attention to $p$ and not to $q$. We also have $(\mathcal{M},w) \vDash B_b B_a(A_a p \land A_a q)$: the AI agent ($b$) thinks that Ann believes to be paying attention to all aspects of the CV.

		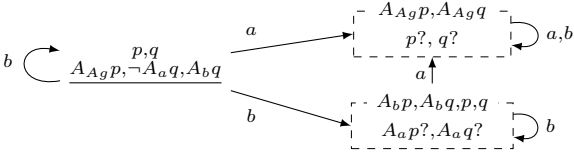
\begin{figure}
			\label{figure:prop-attention-first-example}
			\begin{center}
				\begin{tikzpicture}[auto]\tikzset{ 
						deepsquare/.style ={rectangle,draw=black, inner sep=2pt, very thin, dashed, minimum height=3pt, minimum width=1pt, text centered}, 
						world/.style={}, 
						actual/.style={},
						cloudy/.style={cloud, cloud puffs=15, cloud ignores aspect, inner sep=0.5pt, draw}
					}
					\hspace{-10pt}
					\node [actual] (!) {$\underline{\s \substack{p, q \\A_{Ag} p, \neg A_aq,
					A_bq}}$};
					\path (!) edge [-latex, looseness=4,in=170,out=190] (!) node [above, xshift=-52pt, yshift=-4pt] {$\s b$};
					
					\node [world, right=of !, xshift=20pt, yshift=10pt] (1) {$\hspace{5mm}\s p?,\hspace{2pt}q?\hspace{5mm}$};
					\node [deepsquare, fit={(1)},yshift=2pt](square) {};
					\node[fill=white] (square name 1) at (square.north) [yshift=0pt]{$\s A_{Ag} p, A_{Ag}q$};
					\path (square) edge[-latex, loop, out=8, in=-8, looseness=4] node {$\s a,b$} (square);
					\path (!) edge [-latex] (square.west) node [above, xshift=40pt, yshift=8pt] {$\s a$};

					\node [world, below=of 1, yshift=7pt] (2) {$\hspace{2.0mm}\s A_ap?,A_aq?\hspace{2.0mm}$};
					\node [deepsquare, fit={(2)},yshift=2pt](square1) {};
					\node[fill=white] (square name 1) at (square1.north) [yshift=0pt]{$\s A_b p, A_b q, p, q$};
					\path (square1) edge[-latex, loop, out=8, in=-8, looseness=3.5] node {$\s b$} (square1);
					
					
					\path (!) edge [-latex] (square1.west) node [above, xshift=40pt, yshift=-25pt] {$\s b$};
					
					
					\path[draw,-latex] ([yshift=4mm]2.north) to node[yshift=-0.7mm,xshift=0.5mm]{$\s a$} (1);
					\node[above=of square1, yshift=-26pt] (anchor) {};
				\end{tikzpicture}
			\end{center}
			
			\caption{ 
				A pointed Kripke model $(\mathcal{M},w)$ for $\mathcal{L}_{\text{PA}}$.  In the figure, $p$ stands for ``the applicant has published several papers in top-tier journals'', $q$ for ``the applicant has made significant contributions to diversity''. There are two agents, Ann ($a$) and the AI agent ($b$), i.e.\ $Ag=\{a,b\}$. We use the following conventions. For a set of agents $Ag'$, 
				$A_{Ag'}p := \bigwedge_{a\in Ag'} A_ap $.
				Worlds are represented either by a sequence of literals true at the world, or by such a sequence of literals where some of the atoms are followed by question marks: $p_1?,\dots,p_n?, \ell(q_1),\dots,\ell(q_m)$.
				This is shorthand for the set of $2^n$ worlds corresponding to all possible truth-value assignments of the atoms $p_1,p_2,...,p_n$, where $\ell(q_1), ...,\ell(q_m)$ are true at each of these worlds.
				When a world appears inside a dashed box, all the literals in the label of that box are also true at the world. The actual world is underlined. The accessibility relations are represented by labelled arrows. An arrow from (or to) the border of a box means that there is an arrow from (or to) all the events inside the box. 
			}\label{figure: propositional attention}
		\end{figure} 
		
\end{example}

Fig.~\ref{figure: propositional attention} models the external perspective on the situation, for simplicity of presentation. To represent the case where it is the AI agent $b$ itself using the logic to reason about the scenario, we would apply perspective shifts and consider the situation from the perspective of $b$ only %
\cite{bolander2021del}. 

\paragraph{Standard event models for propositional attention}
The propositional event models of \citeauthor{belardinelli2023attention}~[\citeyear{belardinelli2023attention}]
are intended to represent the 
revelation
of concurrent stimuli from the environment, of which an agent may pay attention to and receive only a portion.\footnote{We 
	choose the term \emph{revelation} over the term \emph{announcement}  
	to emphasize that the information is not necessarily disclosed by an agent (announced), but may 
	simply 
	be something that is seen or heard~\cite{ditmarsch2023announced}. Revelations are here always truthful.} 
Accordingly, a revelation is represented as a conjunction of propositional literals, that is, the conjunction of facts being revealed by the relevant stimuli. An agent may pay attention to only some aspects of the revealed stimuli and, hence, receive only some of these facts.
Say agents $a$ and $b$ are shown a picture of $c$ wearing a blue hat and a red shirt. Use $p_1$ for ``$c$ is wearing a blue hat'' and $p_2$ for ``$c$ is wearing a red shirt''. Then the formula revealed by 
the picture is $p_1 \land p_2$, with the intuition that 
$p_1 \land p_2$ \emph{contains} the two pieces of information $p_1$ and $p_2$, and that $a$ and $b$ may attend to and receive different parts of it (e.g.\ agent $a$ may only pay attention to the hat, 
$p_1$, and agent $b$ only to 
the shirt, 
$p_2$). 


We use the following additional conventions. For a set of formulas $S$, 
$\bigwedge S$ denotes their conjunction. For $S=\emptyset$, define $\bigwedge S := \top$. Conjunctions of literals are assumed to be in a normal form where each atom occurs at most once and in a specific order.
A conjunction of literals $\varphi = \bigwedge_{1 \leq i \leq n} \ell_i$ is said to \emph{contain} the literals $\ell_1,\dots,\ell_n$, and we write $\ell_i \in \varphi$. 

The next definition recalls the propositional attention event models~\cite{belardinelli2023attention}. The function $id_E$ used in the definition is the identity function on $E$, that is, for every $e\in E, id_E(e)=e$.
\begin{definition}[Event model for propositional attention $\mathcal{F}(\varphi)$ \cite{belardinelli2023attention}]\label{e-varphi} 
Let $\varphi=\ell(p_1)\wedge \dots \wedge \ell(p_n)$ with all $p_i \in P$. 
The \emph{event model for propositional attention} representing the revelation of $\varphi$ is    
the multi-pointed event model $\mathcal{F}(\varphi)=((E,Q,id_{E}), E_d)$ for $\mathcal{L}_\text{PA}$ defined by:
\begingroup
\setlength{\abovedisplayskip}{0pt}
\setlength{\belowdisplayskip}{5pt}
\begin{multline*}
	\hspace{-8pt}		E=\{\bigwedge_{p \in S} \ell(p) \wedge  \bigwedge_{a \in Ag} \bigl(\bigwedge_{p \in X_a} A_ap \wedge \bigwedge_{p \in S \setminus X_a} \neg A_ap \bigr)  \colon \\[-0mm] S \subseteq \mathit{At}(\varphi) \text{ and for all } a \in Ag,X_a \subseteq S \}
\end{multline*}
\endgroup

\noindent$Q_a$ is such that $(e,f)\in Q_a$ iff the following holds for all $p$:

\vspace{-1mm}
\begin{itemize}\setlength{\itemsep}{-2pt} 
	\item[-] \textsc{Attentiveness}: if $A_ap\in e$ then $A_ap, \ell(p)\in f$;
	\item[-] \textsc{Inertia}: if $A_ap\notin e$ then $\ell(p)\not\in f;$
\end{itemize}  

\vspace{-1mm} 
$E_d=\{\psi\in E\colon \ell(p)\in \psi, \text{ for all } \ell(p)\in \varphi\}$. 
\end{definition}

This definition is exemplified in Fig.~\ref{fig2:comparison-left}. 
As earlier shown, these event models grow exponentially both in the number of revealed literals and in the number of agents~\cite{belardinelli2023attention}. We will now 
introduce \emph{edge-conditioned event models} that will help us represent these models more compactly, and generalize them more easily. 

\pgfdeclarelayer{background}
\pgfdeclarelayer{foreground}
\pgfsetlayers{background,main,foreground}

\begin{figure*}
\begin{minipage}[c]{0.5\textwidth}
				\begin{center}
					\resizebox{310pt}{180pt}{\begin{tikzpicture}\tikzset{deepsquare/.style ={rectangle,draw=black, inner sep=2.5pt, very thin, dashed, minimum height=3pt, minimum width=1pt, text centered}, 
								world/.style={node distance=12pt,inner sep=1pt},auto,align=center,
								atransitive/.style={-,black}, a/.style={-latex,black}, 
								b/.style={-latex,dotted, thick}, 
								aloopy/.style={-latex, black, loop, in=60, out=20, looseness=3}, 
								bloopy/.style={-latex, dotted, loop, in=45, out=10, looseness=3, thick}}
							\hspace{-30pt}
							
							
							\node [world] (1) {$\s  A_a p,  A_aq$};
							\node [world, below=of 1, yshift=6pt] (2) {$\s  A_a p, \neg A_aq$};
							\node [world, below=of 2, yshift=6pt] (3) {$\s \neg A_a p,  A_aq$};
							\node [world, below=of 3, yshift=6pt] (4) {$\s \neg A_a p, \neg  A_aq$};
							\node [deepsquare, fit={($(1) +(0,4mm)$)(2)(3)(4)}] (square1) {};
							\node[fill=white] (square name 1) at (square1.north) {$\s  A_bp,  A_bq$};
							\node [world, right=of 1, xshift=25pt] (5) {$\s  A_a p,  A_aq$};
							\node [world, below=of 5, yshift=6pt] (6) {$\s  A_a p, \neg A_aq$};
							\node [world, below=of 6, yshift=6pt] (7) {$\s \neg A_a p, A_aq$};
							\node [world, below=of 7, yshift=6pt] (8) {$\s \neg A_a p, \neg  A_aq$};
							\node [deepsquare, fit={($(5) +(0,4mm)$)(6)(7)(8)}] (square 2) {};
							\node[fill=white] (square name 2) at (square 2.north) {$\s {\neg A_bp,  A_bq}$};
							
							\node [world, right=of 5, xshift=25pt] (9) {$\s  A_a p,   A_aq$};
							\node [world, below=of 9, yshift=6pt] (10) {$\s  A_a p, \neg A_aq$};
							\node [world, below=of 10, yshift=6pt] (11) {$\s \neg A_a p,  A_aq$};
							\node [world, below=of 11, yshift=6pt] (12) {$\s \neg A_a p,  \neg  A_aq$};
							\node [deepsquare, fit={($(9) +(0,4mm)$)(10)(11)(12)}] (square 3) {};
							\node[fill=white] (square name 3) at (square 3.north) {$\s { A_bp,  \neg A_bq}$};
							
							\node [world, right=of 9, xshift=25pt] (13) {$\s  A_a p,  A_aq$};
							\node [world, below=of 13, yshift=6pt] (14) {$\s  A_a p, \neg A_aq$};
							\node [world, below=of 14, yshift=6pt] (15) {$\s \neg A_a p,  A_aq$};
							\node [world, below=of 15, yshift=6pt] (16) {$\s \neg A_a p, \neg  A_aq$};
							\node [deepsquare, fit={($(13) +(0,4mm)$)(14)(15)(16)}] (square 4) {};
							\node[fill=white] (square name 4) at (square 4.north) {$\s {\neg A_bp, \neg A_bq}$};
							
							\node[above=of 1, yshift=-18pt] (anchor1) {};
							\node [deepsquare, inner sep=5pt,  fit={(square1)(square 2)(square 3)(square 4)(anchor1)}] (outer) {};
							\node[fill=white] (outer name) at (outer.north) {$\s \underline{p, q}$};
							
							\node [world, below=of 8, yshift=-20pt, xshift=-20pt] (17) {$\s  A_b p$};
							\node [world, below=of 17, yshift=6pt] (18) {$\s\neg A_bp$};
							\node [deepsquare, fit={(17)(18)}, inner sep=6.5pt] (square 5up) {};
							\node[fill=white] (square name 5) at (square 5up.north) {$\s A_ap$};
							
							\node [world, below=of 18, yshift=-9pt] (19) {$\s  A_bp$};
							\node [world, below=of 19, yshift=6pt] (20) {$\s \neg  A_bp$};
							\node [deepsquare, fit={(19)(20)}, inner sep=7pt] (square 5down) {};
							\node[fill=white] (square name 5) at (square 5down.north) {$\s\neg A_ap$};
							
							\node [deepsquare, fit={($(17) +(0,3mm)$)(18)(19)(20)}, inner sep=11.5] (square 5) {};
							\node[fill=white] (square name 5) at (square 5.north) {$\s p$};
							
							\node [world, right=of 17, xshift=75pt] (21) {$\s  A_bq$};
							\node [world, below=of 21, yshift=6pt] (22) {$\s\neg A_bq$};
							\node [deepsquare, fit={(21)(22)}, inner sep=6.5pt] (square 6up) {};
							\node[fill=white] (square name 6) at (square 6up.north) {$\s A_aq$};
							
							\node [world, below=of 22, yshift=-9pt] (23) {$\s   A_bq$};
							\node [world, below=of 23, yshift=6pt] (24) {$\s \neg  A_bq$};
							\node [deepsquare, fit={(23)(24)}, inner sep=7pt] (square 6down) {};
							\node[fill=white] (square name 6) at (square 6down.north) {$\s\neg A_aq$};
							
							\node [deepsquare, fit={($(21) +(0,3mm)$)(22)(23)(24)}, inner sep=11.5] (square 6) {};
							\node[fill=white] (square name 6) at (square 6.north) {$\s q$};
							
							\node [world, below=of 20, yshift=-7pt, xshift=51pt] (25) {$\s \top$};
							
							\path[a] (1) edge [-] (5) node [above, xshift=31pt, yshift=-1pt] {};
							\path[a] (5) edge [-] (9) node [above, xshift=31pt, yshift=-1pt] {};
							\path[a] (9) edge [-] (13) node [above, xshift=31pt, yshift=-1pt] {};
							\draw[atransitive] (1.north east) to [looseness=1.3, bend left] (9.north west);
							\draw[atransitive] (1.north east) to [looseness=1.3, bend left] (13.north west);
							\draw[atransitive] (5.north east) to [looseness=1.3, bend left] (13.north west);
							
							\path[b] (square1) edge [looseness=5,in=115,out=130] node [yshift=-6pt]{}(square1);
							
							\draw[aloopy] (1) to (1);
							\draw[aloopy] (5) to (5);
							\draw[aloopy] (9) to (9);
							\draw[aloopy] (13) to (13);
							
							\path[b] (25) edge [looseness=8,in=85,out=155] node [yshift=-6pt]{} (25);
							\path[a] (25) edge [looseness=7,in=95,out=155] node [yshift=-6pt]{} (25);

							\draw[a] (2.west) to [bend right, out=-90,in=225,looseness=1.4] (square 5up.west);
							\draw[a] (3.west) to [bend right, out=-120,in=170,looseness=1.2] (square 6up.north west);
							\draw[a]   (4.south) to [bend right,in=240,looseness=1.5, out=280] (25.west);
							
							\draw[a] (6.west) to [bend right, out=330,in=220,looseness=1.15] (square 5up.west);
							\draw[b] (square 2.south) to [bend right, in=200,looseness=0.5] node[yshift=-1pt, xshift=-3] { } (21.west);
							\draw[b] (square 2.south) to [bend right, in=200,looseness=0.5] node[yshift=-1pt, xshift=-3] { } (23.west);
							\draw[a] (8.west) to [bend left, out=290,in=255,looseness=2] (25.west);
							\draw[a] (7.east) to [out=-35,looseness=1] (square 6up.north west);
							
							\draw[a] (10.west) to [bend right, out=20,in=160,looseness=1.15] (square 5up.north east);
							\draw[a] (11.east) to [bend right, out=90,in=100,looseness=1] (square 6up.east);
							\draw[a] (12.east) to [bend right, out=70,in=110,looseness=2] (25.east);
							\draw[b] (square 3.south) to [bend right, in=170, out=20, looseness=0.7]  (17.east);		
							\draw[b] (square 3.south) to [bend right, xshift=1pt, in=130, out=15, looseness=1]  (19.east);		
							
							\draw[a] (14.east) to [bend right, out=115,in=180,looseness=1.5] (square 5up.north east);
							\draw[a] (15.east) to [bend left, out=90,in=140,looseness=1] (square 6up.east);
							\draw[a]   (16.south) to [bend right,in=120,out=80, looseness=1.5](25.east);
							\draw[b] (square 4.south) to [bend right, xshift=1pt, in=125, out=75, looseness=1.5]  (25.east);		
							
							\draw[-latex, a, loop, in=60, out=40, looseness=7] (square 5up) to (square 5up);
							\draw[-latex, b, loop, in=45, out=10, looseness=7] (17) to (17);
							\path [b] (17.east) edge [-, bend left, out=90, in=90, looseness=1.5] (19.east);	
							\draw[a] (square 5down.east) to [bend left] (25.north);		
							\draw[b] (18.west) to [bend right, out=270, in=238, looseness=2] (25.west);		
							\draw[-latex, b, loop, in=40, out=10, looseness=8] (19) to (19);
							
							\draw[b] (20.south) to [bend right, looseness=1]  (25.west);	
							
							\draw[-latex, a, loop, in=60, out=40, looseness=7] (square 6up) to (square 6up);
							\draw[-latex, b, loop,in=41, out=10, looseness=8] (21) to (21);
							\draw[-latex, b, loop, in=37, out=10, looseness=9] (23) to (23);
							\path [b] (21.east) edge [-, bend left, out=90, in=90, looseness=1.5] (23.east);			
							\draw[a] (square 6down.west) to [bend right] (25.north);		
							\draw[b] (22.east) to [bend left, out=90, in=124, looseness=2] (25.east);		
							\draw[b] (24.south) to [bend left, looseness=1] (25.east);		
					\end{tikzpicture}} 
		\end{center}
	\vspace{-4mm}
	\caption{Event model $\mathcal{F}(p\land q)$ from \protect\cite{belardinelli2023attention}. 
	 Solid edges are for Ann ($a$), dotted for the AI agent ($b$). The figure adopts the same conventions as in the cited paper: An event is represented by a list of literals, corresponding to (some of the) literals that appear in the conjunctive precondition of the event itself. The formulas in the label of a dashed box are to be included as conjuncts in the precondition of all events inside the box. The convention for the edges is the same as in Fig.~\ref{figure: propositional attention}. See \protect\cite{belardinelli2023attention} for a detailed explanation of the figure.}\label{fig2:comparison-left}
\end{minipage}
\hfill
\begin{minipage}[c]{0.45\textwidth}
		\begin{center}
			\scalebox{1}
			{	\begin{tikzpicture}
					\tikzset{scale=2, 
						world/.style={}, 
						actual/.style={}
					}
					\node [actual,fill=white] (!) at (0,0) {$\s \underline{p \land q}$}; 
					\path (!) edge [-latex, looseness=6,in=70,out=110] (!) node [above, xshift=9pt, yshift=17pt] {$\s i:(A_i p \land A_i q, A_i p \land A_i q)$};
					
					\node [world,fill=white] (1) at (-1,-1) {$\s p$};
					\begin{scope}[on background layer]
						\draw[-latex] (!) arc 
						[
						start angle=90,
						end angle=173,
						x radius=1,
						y radius =1
						] node[midway,left,xshift=-2pt] {$\s i:(A_i p \land \neg A_i q, A_i p)$};
					\end{scope}
					
					\path (1) edge [-latex, looseness=9,in=160,out=200] node [left] {$\s i:(A_i p, A_i p)$} (1);
					
					\node [world,fill=white] (2) at (1,-1) {$\s q$};
					
					\begin{scope}[on background layer]
						\draw[-latex] (!) arc 
						[
						start angle=90,
						end angle=7,
						x radius=1,
						y radius =1
						] node[midway,right,xshift=2pt] {$\s i:(A_i q \land \neg A_i p, A_i q)$};
					\end{scope}
					
					\path (2) edge [-latex, looseness=9,in=20,out=-20] node [right] {$\s i:(A_i q, A_i q)$} (2);

					\node [world] (T) at (0,-2) {$\s \top$};
					\begin{scope}[on background layer]
						\draw[-latex] (1) arc 
						[
						start angle=-180,
						end angle=-95,
						x radius=1,
						y radius =1
						] node[midway,left,xshift=0mm] {$\s i:(\neg A_i p, \top)$};
						
						\draw[-latex] (2) arc 
						[
						start angle=0,
						end angle=-85,
						x radius=1,
						y radius =1
						] node[midway,right,xshift=1mm] {$\s i:(\neg A_i q, \top)$};    
					\end{scope}
					
					\path (T) edge [-latex, looseness=6,in=-110,out=-70] node [below,xshift=9pt] {$\s i:(\top, \top)$} (T);
					
					\path (!) edge[-latex,bend right=0] node[fill=white,above,yshift=-7pt,align=left,xshift=4pt]{$\s i:(\neg A_i p \land \neg A_i q, \top)$}
					(T);

			\end{tikzpicture}} 
		\end{center} 
\vspace{-2.5mm} 
\caption{Edge-conditioned event model for propositional attention $\mathcal{H}(p\land q)$. Events are represented by conjunctive formulas corresponding to the event's own precondition. When for all agents $i\in Ag$, we have a (conditioned) edge $(e{:}\varphi_i,f{:}\psi_i) \in Q_i$, we add an arrow from $e$ to $f$ labelled by $i{:}(\varphi_i,\psi_i)$. This means that agent $i$ has an edge from $e$ to $f$ with source condition $\varphi_i$ and target condition $\psi_i$. For example, the arrow from event $p\wedge q$ to event $p$ labelled by $i{:}(A_ip\wedge\neg A_iq,A_ip)$ corresponds to the edge $(p\wedge q{:} A_ip\wedge\neg A_iq,p{:}A_ip) \in Q_i$, for all $i\in Ag$. This edge models an agent who paid attention to $p$, but not to $q$, and therefore only learns $p$ and that she paid attention to $p$.
}\label{fig2:comparison-right}
\end{minipage}
\hfill
\end{figure*}

\section{Edge-Conditioned Event Models}


The idea behind edge-conditioned event models is to make the edges of event models conditional on formulas. For standard event models, $(e,f)\in Q_a$ means that event $f$ is accessible from event $e$ by agent $a$, and whenever that is the case, we draw an $a$-edge from $e$ to $f$.
In edge-conditioned event models, what is accessible at an event has become conditioned by formulas: $(e{:}\varphi, f{:}\psi)\in Q_a$ means that $f$ is accessible from $e$ by $a$ under 
the condition that $\varphi$ is the case at the source $e$ and $\psi$ is the case at target $f$. As we will see, this simple modification 
has rather advantageous consequences, as it allows  us 
to represent event models significantly more 
succinctly.


\begin{definition}[Edge-conditioned event models]\label{def:edge-conditioned-event-model}
An \emph{edge-conditioned event model} for $\mathcal{L}$ is a tuple $\mathcal{C}=(E,Q,pre)$ where $E$ and $pre$ 
are standard (i.e.\ as in Def. \ref{def: standard event model}), and where $Q: Ag\rightarrow  \mathcal{P}(E\times \mathcal{L}\times E\times \mathcal{L})$ assigns to each agent a set of quadruples $(e,\varphi,f,\psi)$. 
For $(e,\varphi,f,\psi)\in Q_a$, we call $(e,\varphi,f, \psi)$ a \emph{conditioned edge}, where $\varphi$ is the \emph{source condition} (at $e$) and $\psi$ is the \emph{target condition} (at $f$).
We often abbreviate 
$(e,\varphi,f,\psi)$ as 
$(e{:}\varphi,f{:}\psi)$ 
to emphasize that it is an edge from $e$ to $f$, where the source $e$ has
condition $\varphi$, and the target $f$ 
condition $\psi$. 
The set of edge-conditioned event models for $\mathcal{L}$ is denoted by $\ecem(\mathcal{L})$.
Where
$e$ is the \emph{actual event}, $(\mathcal{C},e)$ is a \emph{pointed edge-conditioned event model}.
\end{definition}

\begin{definition}
[Edge-conditioned product update]
\label{def:edge-conditioned-product-update}
Let $\mathcal{M} = (W,R,V)$ be a Kripke model and $\mathcal{C} = (E,Q,pre)$ an edge-conditioned event model, both 
for  the same language $\mathcal{L}$.  
The \emph{product update} of $\mathcal{M}$ with $\mathcal{C}$ is $\mathcal{M} \otimes \mathcal{C} = (W',R',V')$ where $W'$ and $V'$ are standard (Def.~\ref{def: product update no post}), and 
$R'_a=\{((w,e),(v,f))\in (W')^2 \colon (w,v)\in R_a \text{ and } \exists \varphi,\psi \in \mathcal{L} \text{ such that } (e{:}\varphi,f{:}\psi)\in Q_a, (\mathcal{M},w)\vDash \varphi \text{ and } (\mathcal{M},v)\vDash \psi\}$.		
A pointed edge-conditioned event model $(\mathcal{C}, e)$ is called \emph{applicable} in a pointed Kripke model $(\mathcal{M},w)$
if 
$(\mathcal{M}, w )\vDash pre(e)$, and 
then the \emph{product update} of $(\mathcal{M},w)$ with $(\mathcal{C},e)$ is the pointed Kripke model $(\mathcal{M},w) \otimes (\mathcal{C},e)= (\mathcal{M} \otimes \mathcal{C}, (w,e))$.
\end{definition}

Edge-conditioned event models were first introduced by Bolander~[\citeyear{bolander2018seeing}], but only with conditions at the source of edges. Adding conditions also at the target is technically straightforward and matches the conventions for generalized arrow updates~\cite{kooi2011generalized}. Despite the technical simplicity, it turns out to have significant advantages, including that this new formalism generalizes and is more succinct than both generalized arrow updates and standard event models, and that we acheive an exponential succinctness result for the event model for propositional attention that would not hold with only source conditions (Theorem~\ref{thm:exponential-succinctness} below).

We define the language of \emph{DEL with edge-conditioned event models} $\mathcal{L}_\text{ECM}$ as the language given by the grammar of $\mathcal{L}_\text{EL}$ extended with the clause $[\mathcal{C}] \varphi$, where $\mathcal{C}$ is a pointed edge-conditioned event model.  
The semantics of $\mathcal{L}_\text{ECM}$ is defined as for $\mathcal{L}_\text{DEL}$ (Def.~\ref{def:satisfaction}), except that the product update in the semantics of $[\mathcal{C}]\varphi$ uses Def.~\ref{def:edge-conditioned-product-update}.

\begin{table} 
\begin{minipage}{\linewidth}
\caption{\label{tab:logic}Axiomatization of DEL with edge-conditions ($\mathcal{L}_\text{ECM}$).}   

\vspace{-2mm} 
\begin{tabular}{l} \toprule
All prop.\ tautologies and
$B_{a}(\varphi\rightarrow \psi)\rightarrow (B_{a}\varphi\rightarrow B_a\psi)$. \\
$[(\mathcal{C},e)] p \leftrightarrow (pre(e)\rightarrow p)$ \\
$[(\mathcal{C},e)] \neg \psi \leftrightarrow( pre(e) \rightarrow \neg [(\mathcal{C},e)]\psi)$ \\
$[(\mathcal{C},e)](\psi\wedge\chi)\leftrightarrow([(\mathcal{C},e)]\psi\wedge[(\mathcal{C},e)]\chi)$\\
$\displaystyle[(\mathcal{C},e)]B_a\psi \leftrightarrow (pre(e)\rightarrow \! \bigwedge_{\mathclap{(e:\chi,f:\chi')\in Q_a}}\hspace{6pt}(\chi\rightarrow B_a(\chi' \rightarrow [(\mathcal{C},f)]\psi)))$\\
From $\varphi$ and $\varphi\rightarrow\psi$, infer $\psi$. 
From $\varphi$ infer $B_a\varphi$.\\
From $\varphi\leftrightarrow \psi$, infer $\chi[\varphi\slash p]\leftrightarrow \chi[\psi\slash p] $ (substitution).
\\
\bottomrule
\end{tabular}
\end{minipage}
\end{table}
\begin{theorem}[Soundness and completeness]\label{thm: soundness and completeness}
Table~\ref{tab:logic} provides a sound and complete axiomatization of DEL with edge-conditioned event models.
\end{theorem}

\begin{proof}[Proof Sketch] 
The first line and last two lines of Table~\ref{tab:logic} provide a sound and complete axiomatization of the logic of the underlying epistemic language $\mathcal{L}_\text{EL}$~\cite{fagin1995reasoning}. Completeness of the logic of the full language $\mathcal{L}_\text{ECM}$ then follows by standard reduction arguments~\cite{ditmarsch2007dynamic}: we have reduction axioms (lines 2--5 of Table~\ref{tab:logic}) for translating any formula involving the $[\mathcal{C}]$ modality into a formula without it, hence reducing the completeness proof to completeness of the underlying epistemic logic. For soundness, we verify the validity of the reduction axiom for the belief modality, the validity of the others being standard.   
\end{proof}

We define $\mathcal{L}_{\text{PA}^+}$ as the language given by the grammar of 
$\mathcal{L}_\text{ECM}$ extended with the clause $\varphi ::= A_a p$, with $A_a p \in H$. 
It is the language that uses edge-conditioned event models and includes the attention atoms. 
Note that $At(\mathcal{L}_{\text{PA}^+}) = At(\mathcal{L}_\text{PA})= P \cup H$. 
We 
exemplify Definition \ref{def:edge-conditioned-event-model}
by defining an edge-conditioned event model for propositional attention that corresponds to the standard event model 
of Definition~\ref{e-varphi}.

\begin{definition}[Edge-conditioned event model for propositional attention] \label{def:edge-conditioned-prop-attention}
Let $\varphi=\ell(p_1)\wedge \dots \wedge \ell(p_n)$ with $p_i \in P$.
The \emph{edge-conditioned event model for propositional attention} representing the revelation of $\varphi$ is the pointed edge-conditioned event model		
$\mathcal{H}(\varphi)=((E,Q,id_E),\varphi)$ for $\mathcal{L}_{\text{PA}^+}$ defined by:

\smallskip
\begingroup
\setlength{\abovedisplayskip}{-7pt}
\setlength{\belowdisplayskip}{-5pt}
$E=\{\bigwedge_{p\in S} \ell(p) \colon S\subseteq At(\varphi)\}$,


\begin{multline*}
Q_a = \{ 	(\bigwedge_{\mathclap{p\in S}}\ell(p){:} \bigwedge_{\mathclap{p\in T}}A_ap \land \bigwedge_{\mathclap{p\in S\setminus T}}\neg A_ap, \bigwedge_{\mathclap{p\in T}}\ell(p){:} \bigwedge_{\mathclap{p\in T}}A_ap):\\[-2pt] T \subseteq S \subseteq At(\varphi) \}.
\end{multline*}
\endgroup
\end{definition}
Figure~\ref{fig2:comparison-right} shows $\mathcal{H}(p \land q)$, i.e.\ the same revelation as in Figure~\ref{fig2:comparison-left} (when $Ag = \{a,b\}$). The standard and the edge-conditioned event models for propositional attention correspond to each other by being
\emph{update equivalent}. 
\begin{definition}[Update equivalence~\cite{kooi2011generalized,eijck2012action}] 
\label{def:update equivalence}
Let $\mathcal{D}$ be a standard or edge-conditioned event model for $\mathcal{L}$ 
and $\mathcal{D}'$ a standard or edge-conditioned event model for $\mathcal{L}'$, where $At(\mathcal{L}) = At(\mathcal{L}')$.
We say that $\mathcal{D}$ is \emph{update equivalent to} $\mathcal{D}'$ if for all Kripke models $\mathcal{M}$ with atom set $At(\mathcal{L})$,
$\mathcal{M}\otimes\mathcal{D}$ and $\mathcal{M}\otimes\mathcal{D}'$ are bisimilar.\footnote{If $\mathcal{M}$ is a Kripke model for $\mathcal{L}$, it is also a Kripke model for any $\mathcal{L}'$ with $At(\mathcal{L}') = At(\mathcal{L})$, since its definition only depends on the set of atoms and agents, and the agent set is fixed. We also refer to such Kripke models as \emph{Kripke models with atom set} $At(\mathcal{L})$. The notion of bisimulation is standard in modal logic \cite{blac.ea:moda}. The definition is given in the Technical Appendix.}
\end{definition}
Update equivalence means semantic equivalence, as bisimilar models agree on all formulas~\cite{blac.ea:moda}. 

\begin{theorem}\label{prop:F-H-equivalence} 
For any conjunction of propositional literals $\varphi$, 
$\mathcal{F}(\varphi)$ and $\mathcal{H}(\varphi)$ are update equivalent. 
\end{theorem}
\begin{proof}[Proof Sketch]	
Let $\varphi = \ell(p_1) \land \cdots \land \ell(p_n)$ and let $\mathcal{M}$ be any Kripke model. We define a bisimulation relation  
from $\mathcal{M}\otimes\mathcal{H}(\varphi) = (W,R,V)$ to $\mathcal{M}\otimes\mathcal{F}(\varphi) = (W',R',V')$ by $Z= \{ ((w,e),(w',e')) \in W \times W' \mid  w = w' \text{ and for all } i = 1,\dots,n: \ell(p_i) \in e \text{ iff } \ell(p_i) \in e' \}$. In other words, we match each event $\bigwedge_{p \in S} \ell(p)$ of $\mathcal{H}(\varphi)$ with all events of $\mathcal{F}(\varphi)$ that are of the form $\bigwedge_{p \in S} \ell(p) \land \cdots$. Then the proof proceeds by showing that $Z$ satisfies the [Atom], [Forth] and [Back] conditions of being a bisimulation relation. 
\end{proof}
This theorem implies that the event models of Figs.~\ref{fig2:comparison-left} and~\ref{fig2:comparison-right} are equivalent, but the latter is clearly simpler and easier to follow. As it turns out, it is also exponentially smaller (Thm.~\ref{thm:exponential-succinctness}).  
Moreover, while $\mathcal{F}(\varphi)$ 
is multi-pointed, 
$\mathcal{H}(\varphi)$
is only pointed. Hence, by using edge-conditioned event models, we achieve several advantages in terms of simplicity and succinctness. There are also conceptual advantages. By separating the preconditions from the edge-conditions, these models introduce a conceptual distinction between the informational content pertaining to the event itself, which belongs to the event's preconditions, and information determining which events are considered possible by an agent, which is now a condition on the accessibility relation. Information about agents' attention is of the latter kind, as it pertains to the agent's perspective on the event (see also \cite{watzl2017structuring}, Ch. 13). 
It is not a part of the information that is revealed, but rather a condition on the accessibility relations (i.e., an edge-condition). Standard event models instead merge these distinct information types into a single construct, namely the preconditions.

\begin{example}\label{ex: updated propositional attention} Continuing Example \ref{ex:static}, Ann has now reviewed the applicant's CV. She has learned that they published several papers in top-tier journals, but has not learned about their contribution to diversity, as she did not pay attention to it. The AI agent still has no information about what Ann paid attention to, and so thinks that Ann may have learned any part of the CV. This situation is depicted in Fig.~\ref{figure: update of propositional attention}, where e.g. $(\mathcal{M}\otimes\mathcal{H}(p\wedge q),(w,e))\vDash B_a p\wedge \neg B_a q \wedge \neg B_a \neg q$: Ann has only learned about the candidate's publications, overlooking their contributions to diversity. We also have $(\mathcal{M}\otimes\mathcal{H}(p\wedge q),(w,e))\vDash \neg B_b A_a p\wedge \neg B_b \neg A_a p \wedge \neg B_b A_a q\wedge \neg B_b \neg A_a q$: the AI agent still has no information about Ann's attention.
\begin{figure}
\begin{center}
\begin{tikzpicture}[auto]\tikzset{ 
		deepsquare/.style ={rectangle,draw=black, inner sep=2.5pt, very thin, dashed, minimum height=3pt, minimum width=1pt, text centered}, 
		world/.style={}, 
		actual/.style={},
		cloudy/.style={cloud, cloud puffs=15, cloud ignores aspect, inner sep=0.5pt, draw}
	}
	
	\node [actual] (!) {$\underline{\s \substack{p, q \\A_a p, \neg A_aq, \\ A_b p, A_bq}}$};
	\path (!) edge [-latex, loop, looseness=4,in=105,out=70] node [above] {$\s b$} (!);
	
	\node [world, right=of !, yshift=10pt] (1) {$\s p,q?$};
	\node [deepsquare, fit={(1)}](square1) {};
	\path (square1) edge [-latex, loop, looseness=4.6,in=110,out=70] (square1) node [right, xshift=2pt, yshift=15pt] {$\s a,b$};
	\path (!) edge [-latex] (square1) node [above, xshift=35pt, yshift=4pt] {$\s a$};
	
	\node [world, right=of 1, xshift=-5pt] (2) {$\s p,q$};
	\path (2) edge [-latex, looseness=6,in=110,out=60] (2) node [right, xshift=5pt, yshift=12pt] {$\s a,b$};

	\node [world, right=of 2, xshift=-5pt] (3) {$\s p?,q$};
	\node [deepsquare, fit={(3)}](square3) {};
	\path (square3) edge [-latex, looseness=4.4,in=110,out=70] (square3) node [right, yshift=15pt, xshift=2pt] {$\s a,b$};
	
	\node [world, right=of 3, xshift=-5pt] (4) {$\s p?,q?$};
	\node [deepsquare, fit={(4)}](square4) {};
	\path (square4) edge [-latex, looseness=4.6,in=110,out=70] (square4) node [right, yshift=15pt, xshift=2pt] {$\s a,b$};
	
	\node[left=of 1, xshift=29pt] (anchor1) {};
	\node[above=of 1, yshift=-20pt] (anchor2) {};
	\node[below=of 1, yshift=31pt] (anchor3) {};
	\node[right=of 4, xshift=-29pt] (anchor4) {};
	\node [deepsquare,fit={(anchor1)(anchor2)(anchor3)(anchor4)(1)(2)(3)(4)}](squareb) {};
	\node[fill=white] (square name 1) at (squareb.north) [yshift=-0pt]{$\s A_{Ag} p, A_{Ag}q$};

	\node [world, below=of 1, yshift=13pt] (1b) {$\s A_ap,\neg A_aq$};
	\path ([yshift=-3pt]1b.north) edge [-latex] (square1.south) node [right, yshift=8pt] {$\s a$};
	
	\node [world, below=of 2, right=of 1b, xshift=-26.5pt] (2b) {$\s A_ap,A_aq$};
	\path ([yshift=-3pt]2b.north) edge [-latex] (2.south) node [right, yshift=8pt] {$\s a$};
	
	\node [world, below=of 3, right=of 2b, xshift=-26.5pt] (3b) {$\s \neg A_ap,A_aq$};
	\path ([yshift=-3pt]3b.north) edge [-latex] (square3.south) node [right, yshift=8pt] {$\s a$};			
	
	\node [world, below=of 3, right=of 3b, xshift=-28pt] (4b) {$\s \neg A_ap,\neg A_aq$};
	\path ([yshift=-3pt]4b.north) edge [-latex] (square4.south) node [right, yshift=8pt] {$\s a$};			
	
	\node [deepsquare,fit={(1b)(2b)(3b)(4b)}](squareb) {};
	\node[fill=white] (square name 1) at (squareb.south) [yshift=-0pt]{$\s p, q, A_b p, A_b q$};
	\path (squareb) edge [-latex, looseness=5,in=182,out=177] (squareb) node [left, xshift=-102pt, yshift=0pt] {$\s b$};
	
	\path (!) edge [-latex] (squareb.north west) node [above, xshift=30pt, yshift=-10pt] {$\s b$};
	
\end{tikzpicture}
\end{center}

\vspace{-3mm}
\caption{The pointed Kripke model $(\mathcal{M}\otimes\mathcal{H}(p\land q),(w,e))$ for $\mathcal{L}_{\text{PA}^+}$, with $\mathcal{H}(p\land q)$ given in Fig.~\ref{fig2:comparison-right}. We use the same conventions as before and omit worlds that are inaccessible by all agents.
}\label{figure: update of propositional attention}
\end{figure}
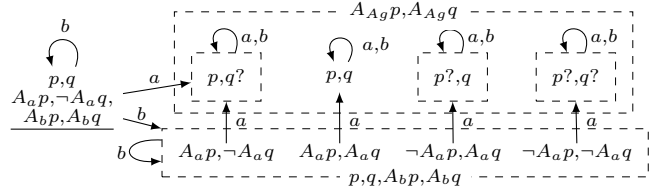 
\end{example}


\section{Expressivity and succinctness} 
Before introducing our logic of general attention, we examine key  
properties 
of edge-conditioned event models,
which may be of independent interest for DEL and its applications.  
\begin{definition}[Transformation of standard event models into edge-conditioned event models]
\label{def: translation SEM into ECEM}
By mutual recursion, we define mappings $T_1: \sem(\mathcal{L}_\text{DEL}) \to \ecem(\mathcal{L}_\text{ECM})$ and
$T_2: \mathcal{L}_\text{DEL} \to \mathcal{L}_\text{ECM}$. 
Define $T_1$ by $T_1(E,Q,pre) = (E,Q',pre')$ where for each $a \in Ag$,  
$Q'_a=\{(e{:}\top,f{:}\top)\colon (e,f)\in Q_a\}$, and for each $e \in E$, $pre'(e) = T_2(pre(e))$. Define $T_2$ by:
\begingroup
\setlength{\abovedisplayskip}{2pt}
\setlength{\belowdisplayskip}{2pt}	
\[
\begin{array}{ll} 
T_2(p) = p, \text{ for }  p \in P &
T_2(\neg \varphi) = \neg T_2 (\varphi) \\
T_2(\varphi \land \psi) = T_2(\varphi) \land T_2(\psi) &
T_2(B_a \varphi) = B_a T_2(\varphi) \\
T_2([(\mathcal{E},e)]\varphi) = [(T_1(\mathcal{E}),e) ] T_2(\varphi) 
\end{array}
\]
\endgroup
\end{definition}	
\begin{theorem}\label{prop: equivalence SEM into EDGY}
$\mathcal{E} \in \sem(\mathcal{L}_\text{DEL})$ 
is update equivalent to
$T_1(\mathcal{E})$.
\end{theorem}
\begin{proof}[Proof Sketch.] We show that for any Kripke model $\mathcal{M}$ and  $\mathcal{E} \in \sem(\mathcal{L}_\text{DEL})$, 
$\mathcal{M}\otimes\mathcal{E}$ is isomorphic to $\mathcal{M}\otimes T_1(\mathcal{E})$, i.e.\ they have the same worlds and accessibility relations. The proof is by induction on the level of $\mathcal{E}$ in the DEL language hiearchy~\cite{kooi2011generalized}, i.e. on the depth of nesting of event model modalities inside $\mathcal{E}$. The proof is straightforward as the translation $T_1$ only (recursively) replaces edges $(e,f)$ with edges $(e{:}\top,f{:}\top)$ having trivial source and target conditions.\footnote{A transformation of standard event models into edge-conditioned event models was first provided by \cite[p. 161]{li2023phdthesis}. Notice that their tranformation is only for event models that do not contain dynamic preconditions, whereas ours includes them too.}
\end{proof}
Theorem~\ref{prop: equivalence SEM into EDGY} shows that any standard event model can be equivalently  represented as an edge-conditioned one. Hence, $\mathcal{L}_\text{ECM}$ is at least as expressive as $\mathcal{L}_\text{DEL}$. It turns out they are equally expressive, as any edge-conditioned event model can also be transformed into an equivalent standard one, 
using a construction inspired by 
Kooi and Renne~%
[\citeyear[Thm 4.7]{kooi2011generalized}]:
\begin{definition}[Transformation of edge-conditioned event models into standard event models]
\label{def: transformation edgys to SEMs}
\label{def: edgies into standard}
For any edge-conditioned event model $(E,Q,pre)$ for $\mathcal{L}_\text{ECM}$ and any $e \in E$, let $\Phi(e)$ be the set of source and target conditions at $e$, i.e.\  $\Phi(e)=\{\varphi\in\mathcal{L}_\text{ECM}\colon (e{:}\varphi,f{:}\psi)\in Q_a\}\cup \{\psi\in\mathcal{L}_\text{ECM}\colon (f{:}\varphi,e{:}\psi)\in Q_a\}$. Set $\Phi'(e)=\Phi(e)\cup \{\neg\varphi\colon \varphi\in \Phi(e)\}$ and let $\mathsf{mc}(e)$ denote the set of maximally consistent subsets of $\Phi'(e)$.
By mutual recursion, we define mappings $T_1': \ecem(\mathcal{L}_\text{ECM}) \to \sem(\mathcal{L}_\text{DEL})$ 
and $T_2':\mathcal{L}_\text{ECM} \to \mathcal{L}_\text{DEL}$. The mapping $T_2'$ is as $T_2$ of Definition \ref{def: translation SEM into ECEM}, except we replace $T_i$ by $T_i'$ and $\mathcal{E}$ by $\mathcal{C}$. We define $T_1'$ by 
$T_1'(E,Q,pre) = (E',Q',pre')$ where:
\begingroup 
\setlength{\abovedisplayskip}{2pt}
\setlength{\belowdisplayskip}{2pt}
\[
\begin{array}{l}
E'= \{(e,\Gamma)\colon e\in E, \Gamma\in \mathsf{mc}(e)\}, \\
Q'_a  =  \{((e,\Gamma),(e',\Gamma'))\in E'\times E' \colon (e{:}\varphi, e'{:}\varphi')\in Q_a, \\
\qquad \quad \varphi\in \Gamma, \varphi'\in \Gamma'\}, \\
pre'((e,\Gamma)) = T_2'(pre(e)\wedge \bigwedge\Gamma).
\end{array}
\]
\endgroup
\end{definition}

\begin{theorem}\label{prop: equivalence EDGY to SEM}
$\mathcal{C} \in \ecem(\mathcal{L}_\text{ECM})$ is update equivalent to $T_1'(\mathcal{C})$.
\end{theorem}
\begin{proof}[Proof Sketch.]Same proof idea as for Theorem~\ref{prop: equivalence SEM into EDGY}, now using techniques from
	Kooi and Renne~[\citeyear[Thm 4.7]{kooi2011generalized}].
\end{proof}
We have now shown equal expressivity between standard event models and edge-conditioned ones, and hence also equal expressivity of $\mathcal{L}_\text{DEL}$ and $\mathcal{L}_\text{ECM}$. Next, we 
show that edge-conditioned event models are more succinct than standard ones (can be exponentially smaller, Thm.~\ref{thm:exponential-succinctness}, and never more than linearly larger, Thm.~\ref{thm:linear-growth-SEM}). We use $|\mathcal{E}|$ for the size of an event model
 $\mathcal{E}$ (standardly defined and straightforwardly generalized to edge-conditioned event models, \if\extendedversion 0 see the extended version on Arxiv\else see appendix\fi). 
\begin{theorem} 
\label{thm:linear-growth-SEM}
For any $\mathcal{E} \in \sem(\mathcal{L}_\text{DEL})$, $T_1(\mathcal{E})$ has size $O(|\mathcal{E}|)$.
\end{theorem}
\begin{proof}[Proof Sketch]
As mentioned in the proof sketch of Theorem~\ref{prop: equivalence SEM into EDGY}, 
$T_1(\mathcal{E})$ is achieved from $\mathcal{E}$ by (recursively) adding $\top$ as both source and target condition on each edge $(e,f)$, and this only gives a linear blowup in size.  
\end{proof}



\begin{theorem}[Exponential succinctness of edge-conditioned event models]
\label{thm:exponential-succinctness}
Let $p \in P$ and let $n = |Ag|$. Then:

\vspace{-1mm} 
\begin{enumerate}\setlength{\itemsep}{0pt}
\item $\mathcal{H}(p)$ is of size $O(n)$.
\item Any standard event model that is update equivalent to $\mathcal{H}(p)$ has at least $2^n$ events.
\end{enumerate}
\end{theorem}

\vspace{-4mm} 
\begin{proof}[Proof Sketch.]
Item 1 follows from inspecting Fig.~\ref{fig2:comparison-right}. 
 To prove 2, we consider any standard event model $\mathcal{E}$ update equivalent to $\mathcal{H}(p)$. For each subset $Ag' \subseteq Ag$, we define a pointed Kripke model $(\mathcal{M}_{Ag'},w_1) = ((W,R,V),w_1)$ with $W = \{w_1,w_2\}$, $R_a = \{(w_1,w_2)\}$ for all $a \in Ag$, $V(w_1) = \{p \} \cup \{A_a p: a \in Ag'\}$ and $V(w_2) = \emptyset$.   
Using Def.~\ref{def:edge-conditioned-product-update} and Def.~\ref{def:edge-conditioned-prop-attention}, we show that the actual world 
of $(\mathcal{M}_{Ag'},w_1) \otimes \mathcal{H}(p)$ has an $a$-edge to a $\neg p$-world iff $a \in Ag \setminus Ag'$. Since $\mathcal{E}$ is update equivalent to $\mathcal{H}(p)$, then also the actual world of $(\mathcal{M}_{Ag'},w_1) \otimes \mathcal{E}$ has an $a$-edge to a $\neg p$-world iff $a \in Ag \setminus Ag'$. From this it follows that $\mathcal{E}$ has an actual event $e_d^{Ag'}$ with an $a$-edge to an event applicable in $w_2$ iff $a \in Ag \setminus Ag'$. For $Ag' \neq Ag''$, the events $e_d^{Ag'}$ and $e_d^{Ag''}$ must hence have different outgoing edges and thus be distinct.  This implies that there are at least as many distinct events in $\mathcal{E}$ as there are subsets of $Ag$, i.e. at least $2^{|Ag|}$. 
  \end{proof}
This result shows that edge-conditioned event models can be exponentially smaller than their standard event model counterparts, and that exponential succinctness specifically holds for the event models for propositional attention (as comparing Figs.~\ref{fig2:comparison-left} and~\ref{fig2:comparison-right} also suggests).


\paragraph{Generalized arrow updates}
The \emph{generalized arrow updates} of Kooi and Renne~[\citeyear{kooi2011generalized}] are a distinct class of event models that are similar to our edge-conditioned event models in including source and target conditions. However, generalized arrow updates do not have preconditions, and so they do not straightforwardly generalize standard event models as our edge-conditioned event models do. They however still have the same expressivity as standard event models~\cite{kooi2011generalized}, and hence the same expressivity as edge-conditioned event models. A disadvantage of generalized arrow updates compared to edge-conditioned event models is that they can be less succinct than their standard event model counterparts~\cite[Thm.~3.14]{kooi2011generalized}. In the
\if\extendedversion 0 extended version of the paper on Arxiv\else appendix\fi, 
we show that edge-conditioned event models are also at least as succinct as generalized arrow updates, via  a result corresponding to Theorem~\ref{thm:linear-growth-SEM}. Therefore, edge-conditioned event models may be a good 
choice of event model formalism for DEL, as they are always at least as succinct as standard event models and generalized arrow updates, and sometimes exponentially more succinct. Furthermore, as they straightforwardly generalize standard event models, extensions such as with postconditions~\cite{ditmarsch2008semantic} become trivial for edge-conditioned event models, which is not true for generalized arrow updates.









\section{A Logic for General Attention}\label{sect:general}
We now generalize propositional attention to account for attention to, and revelation of, arbitrary formulas. 
The  \emph{language of general attention} $\mathcal{L}_\text{GA}$ is the language given by the grammar of $\mathcal{L}_\text{ECM}$ extended with the clause $\varphi ::= A_a \varphi$, where $a \in Ag$ and $A_a$ is a new modal operator.    
The formula $A_a \varphi$ reads ``agent $a$ pays attention to $\varphi$''. Note that, while in $\mathcal{L}_\text{PA}$ the formula $A_a p$ is an atom, in $\mathcal{L}_\text{GA}$ it is a modality applied to a propositional atom.  
Additionally, in $\mathcal{L}_\text{PA}$, the formula $A_a p$ reads ``agent $a$ is paying attention to \emph{whether} $p$'', whereas in $\mathcal{L}_\text{GA}$ it reads ``agent $a$ is paying attention \emph{to} $p$''. 


Moving to general attention allows to formalize many new scenarios, namely all those where agents attend to more complex stimuli than just conjunctions of literals. For example, we may have $A_a ((p\vee q)\rightarrow r)$, meaning that agent $a$ is paying attention to the conditional $(p\vee q)\rightarrow r$. Such a conditional may represent the statement of a mathematical theorem, and 
$A_a ((p\vee q)\rightarrow r)$ then says that agent $a$ is paying attention to it. 
Following \citeauthor{belardinelli2023attention}~[\citeyear{belardinelli2023attention}], we understand attention as being directed to truthful revelations.  
Attending to a theorem then means 
that, if the theorem is revealed, maybe as part of a larger revelation such as a research talk,
the agent will learn its truth value. 

Another application is one where agents may pay attention (or not) to the utterances of other agents.
Say that agent $a$ only pays attention to what agent $b$, but not agent $c$, says about $p$. 
In DEL, the truthful and public announcement of a formula $\varphi$ by an agent $i$ can be represented by the singleton event model where the actual event has precondition $B_i \varphi$ 
\cite{ditmarsch2023announced}. Such an announcement makes all agents know that $i$ believes $\varphi$. Then, to formalise the mentioned attention situation, 
we could use the formula $A_a B_b p\wedge A_a B_b \neg p \wedge \neg A_a B_c p \wedge \neg A_a B_c \neg p$: 
if agent $b$ truthfully announces the (believed) truth-value of $p$, then agent $a$ receives that announcement, but if agent $c$ does the same, 
$a$ receives nothing. 
Other scenarios can be modeled, such as attention to the attention of other agents. For example, the formula $A_a A_b p$ 
represents that $a$ pays attention to $b$ paying attention to $p$.
Besides added expressivity, it is also conceptually natural to treat attention as a modality, similarly to propositional attitudes such as belief, intention, and awareness.  
The language $\mathcal{L}_\text{GA}$ is interpreted in attention models:
\begin{definition}[Attention model]
\label{def: attention model}
An \emph{attention model} is a tuple $\mathcal{M} = (W,R,V,\mathcal{A})$ where $(W,R,V)$ is a Kripke model for $\mathcal{L}_\text{GA}$ 
and $\mathcal{A}: Ag \times W \to \mathcal{P}(\mathcal{L}_\text{GA})$ is an \emph{attention function}. 
For an \emph{actual world} $w$, $(\mathcal{M},w)$ is a \emph{pointed attention model}. 
\end{definition}
The notion of edge-conditioned product update (Def.~\ref{def:edge-conditioned-product-update}) immediately extends to attention models by defining $\mathcal{A}(a,(w,e)) = \mathcal{A}(a,w)$ for all $(w,e)$ of the updated model. 
This simply means that each world preserves its attention assignments, similarly to how we treat the valuation function in the product update.
The set of formulas that agent $a$ is paying attention to at world $w$ is the \emph{attention set} $\mathcal{A}(a,w)$, also denoted by $\mathcal{A}_a(w)$.
The truth of $\mathcal{L}_\text{GA}$ formulas is defined by the same clauses as for $\mathcal{L}_\text{ECM}$ with the following addition: 
$(\mathcal{M},w) \vDash A_a \varphi \text{ iff }  \varphi \in \mathcal{A}_a(w)$. 

This setting has clear similarities with the logic of general awareness \cite{fagin1988belief}. This does not mean that our framework reduces to a formalism for awareness, as the crucial aspect of attention separating it from awareness is that attention determines what agents learn, as we will see, which awareness does not.\footnote{Awareness is a static notion, introduced to restrict the formulas that agents may reason and have explicit beliefs about, as a solution to the logical omniscience problem  \cite{fagin1988belief}.}
Yet, the two frameworks are equivalent in terms of static language and models. We also do not place any restriction on agents' attention sets, similarly to what happens in the logic of general awareness. We may have that $\varphi\wedge \psi\in\mathcal{A}_a(w)$ but $\psi \wedge \varphi \not\in\mathcal{A}_a(w)$, or we may have that
$\varphi, \psi  \in\mathcal{A}_a(w)$ but $\varphi\wedge \psi\not\in\mathcal{A}_a(w)$. Whether this freedom is reasonable depends on the specific applications one has in mind, which may require to impose a range of closure properties on $\mathcal{A}_a$, called \emph{attention principles}. Consider the following examples and suggested applications:


\textbf{Conjunctive closure}: \emph{$\varphi\wedge\psi\in \mathcal{A}_a(w)$ iff $\varphi\in \mathcal{A}_a(w)$ and $\psi\in \mathcal{A}_a(w)$}.
This principle makes sense when representing `divided attention': every time an agent attends to $\varphi \wedge \psi$ it is as if she is dividing attention between $\varphi$ and $\psi$ (i.e.\ she attends to $\varphi$ and attend to $\psi$ separately) and vice versa.
This principle may not be natural in resource-limited settings, where it may be possible to attend to $\varphi$ and $\psi$ separately but not together.  


\textbf{Commutativity}:  \emph{$\varphi\wedge \psi\in \mathcal{A}_a(w)$ iff  $\psi\wedge\varphi\in \mathcal{A}_a(w)$}. 
We may treat the order of conjuncts 
as irrelevant, e.g.\ when abstracting from the temporal order of information presented.

\textbf{Sublanguage closure}: \emph{If $\varphi\in \mathcal{A}_a(w)$, and $\psi$ is a formula constructed from
atoms appearing in $\varphi$ (i.e.\ $At(\psi)\subseteq At(\varphi)$), then $\psi\in \mathcal{A}_a(w)$}. This principle may represent agents who are interested in a specific issue, and so their attention focuses on anything that talks about it.

\textbf{Subformula closure}: \emph{If $\varphi\in \mathcal{A}_a(w)$ and $\psi$ is a subformula of $\varphi$, then $\psi\in \mathcal{A}_a(w)$}. 
This principle is justified when modeling information that is such that agents cannot pay attention to it unless they attend to all its components. 


\textbf{Agent $a$ ignoring agent $b$}: \emph{$B_b\varphi \not\in\mathcal{A}_a(w)$ for all $\varphi$ (or all $\varphi$ that are about a certain issue)}. 
This can model agents who  
systematically do not pay attention to the utterances of other agents
, as in attention-driven social biases \cite{munton2023prejudice}.

\textbf{Agent $a$ attending to agent $b$}: \emph{$B_b\varphi \in\mathcal{A}_a(w)$ for all $\varphi$}. This models agents who systematically pay attention to the utterances of other agents.
We may also model agents who engage in social attention by adding to $\mathcal{A}_a(w)$ formulas involving nested belief or attention modalities, and e.g.\ study social learning \cite{rendell2010copy}.

\textbf{Attention introspection}: \emph{$(w,v) \in R_a$ implies $\mathcal{A}_a(w) = \mathcal{A}_a(v)$}. This principle models agents who have no doubts regarding what they pay attention to. It makes sense to assume it, e.g. when agents are deliberately focusing on something.

\smallskip
Notice that each property discussed above can be easily turned into an axiom schema, and each axiom schema can be used in an axiomatization of the modelled attention notion. As a few examples, the axiom schema $A_a \varphi \wedge A_a \psi \leftrightarrow A_a (\varphi \wedge \psi)$ represents conjunctive closure; $A_a (\varphi \wedge \psi)\leftrightarrow A_a (\psi\wedge \varphi)$ represents commutativity; 
the formulas $A_i\varphi \rightarrow B_iA_i\varphi$ and
$\neg A_i\varphi \rightarrow B_i\neg A_i\varphi$ represent attention introspection. Each of the considered principles corresponds to an axiom schema (or a combination of axioms).
 While many of these principles and their corresponding axioms are discussed in the awareness literature \cite{fagin1988belief}, in this settings they are interpreted differently, as attention is an intrinsically dynamic notion.
Going beyond them, the correspondence to awareness logic allows to study the relation between attention and awareness, rarely discussed in the literature (but see \cite{fritz2016standard,belardinelli2024phdthesis}).


\paragraph{Event models for general attention}
Suppose given a set 
$\Gamma\subseteq \mathcal{L}_\text{GA}$: the formulas that are revealed (or announced) by the occurring event. 
The intuition is now that every agent learns the subset of $\Gamma$ that they are paying attention to.
If $\psi \in \Gamma \cap \mathcal{A}_a(w)$, then agent $a$ learns $\psi$ at world $w$.  

\begin{definition}[Event model for general attention $\mathcal{R}(\Gamma)$]
\label{def: edge-conditioned-general-attention}
Let  $\Gamma\subseteq \mathcal{L}_\text{GA}$ be a set of \emph{revealed formulas}. 
The \emph{event model for general attention} representing the revelation of 
$\Gamma$ is the pointed edge-conditioned event model $\mathcal{R}(\Gamma) = ((E,Q,id_E),\bigwedge \Gamma)$ for $\mathcal{L}_{\text{GA}}$ 
defined by:

\smallskip
\begingroup
\setlength{\abovedisplayskip}{2pt}
\setlength{\belowdisplayskip}{2pt}
$E = \{\bigwedge S\colon S\subseteq \Gamma\}$,

\vspace{-12pt}
\begin{multline*}
Q_a=\{(\bigwedge S{:} \!\bigwedge_{\varphi\in T} \!\!A_a \varphi  \land \hspace{2pt}\bigwedge_{\mathclap{\varphi\in S\setminus T}} \hspace{2pt}\neg A_a\varphi,\bigwedge T{:} \!\!\bigwedge_{\varphi\in T} \!\!A_a \varphi) \colon \\[-0.7mm]
T\subseteq S \subseteq \Gamma \}.
\end{multline*}
\endgroup
\end{definition}

This event model contains, for each subset $S$ of the revealed $\Gamma$, an event $\bigwedge S$. This represents the subset that an agent may learn by paying attention.
The intuition behind the tuple $(\bigwedge S{:} \bigwedge_{\varphi\in T}A_a \varphi  \land \bigwedge_{\varphi\in S\setminus T} \neg A_a\varphi,\bigwedge T{:} \bigwedge_{\varphi\in T}A_a \varphi) $ is that if $\bigwedge S$ is revealed at an event and agent $a$ pays attention only to the subset $T$ of
$S$, then at that event agent $a$ believes that only $\bigwedge T$ was revealed (as the event accessible from $\bigwedge S$ has precondition $\bigwedge T$), and $a$ also learns that she was paying attention to $T$ (as the event accessible from $\bigwedge S$ has (target) condition 
$\bigwedge_{\varphi\in T} A_a\varphi$). This is consistent with the intuition underlying propositional attention, but in a much more general setting. 
Lastly, notice that in the event model for general attention, $\bigwedge \Gamma$ is the actual event, representing the truthful revelation of $\Gamma$. 

Notice that the attention introspection property is immediately preserved by this event model, as the target condition $A_a\varphi$, for all $\varphi\in T$, ensures that an agent knows what she attends to. The other attention principles are also preserved, as the dynamics do not modify agents' attention set. 
%
%

\begin{example}\label{ex: updated general attention}
\label{ex: general attention} 
Consider an enrichment of Example \ref{ex: updated propositional attention} and Fig.~\ref{figure: update of propositional attention}, where the AI now also pays  attention to Ann's utterances regarding the CV. In particular, the AI agent is focusing on whether Ann has information regarding $p$ and $q$, as that would mean that she paid attention to all aspects of the CV. More formally, we are considering a pointed attention model $(\mathcal{M}',w')$ for $\mathcal{L}_\text{GA}$ that is entirely like $(\mathcal{M}\otimes\mathcal{H}(p\land q))$ from Fig.~\ref{figure: update of propositional attention}, except that we have $A_bB_ap\wedge A_b\neg B_ap \wedge A_bB_aq\wedge A_b\neg B_aq$ true at all $b$-accessible worlds (see \if\extendedversion 0 extended version for corresponding figure\else the appendix for the corresponding figure\fi).
Continuing the example, Ann has now submitted an online report to the AI agent with fillable text fields for research qualifications, teaching qualifications, contributions to diversity, etc. Ann writes $p$ for research qualifications (cf.\ Fig.~\ref{figure: propositional attention}), but nothing for diversity, corresponding to the revelation of $\Gamma = \{ B_a p, \neg B_a q, \neg B_a \neg q\}$.    
From this, 
the AI agent correctly infers that Ann has not paid attention to diversity issues. 
This situation is represented in
Fig.~\ref{figure: updated general attention}, where $(\mathcal{M}'\otimes\mathcal{R}(\Gamma),(w',e))\vDash A_ap\wedge \neg A_a q\wedge B_b A_ap\wedge B_b \neg A_a q$: the AI agent correctly believes that Ann has not paid attention to all parts of the CV.
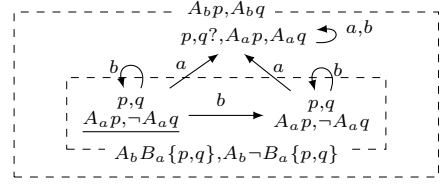
\begin{figure}
\begin{center}
\begin{tikzpicture}\tikzset{ 
deepsquare/.style ={rectangle,draw=black, inner sep=2.5pt, very thin, dashed, minimum height=3pt, minimum width=1pt, text centered}, 
world/.style={}, 
actual/.style={},
cloudy/.style={cloud, cloud puffs=15, cloud ignores aspect, inner sep=0.5pt, draw}
}

\node [actual] (!) {$\underline{\s \substack{p,q \\A_a p, \neg A_aq}}$};
\path (!) edge [-latex, loop, looseness=4.5,in=105,out=70] node [left, xshift=-2pt] {$\s b$} (!);

\node [world, right=of !, xshift=-35pt, yshift=30pt] (1) {$\s p,q?, A_a p, A_aq$};
\path (1) edge [-latex, looseness=5,in=5,out=355] (1) node [right, xshift=35pt, yshift=3pt] {$\s a,b$};
\path (!) edge [-latex] (1) node [above, xshift=19pt, yshift=14pt] {$\s a$};

\node [world, right=of !] (1b) {$\s \substack{p,q\\ A_ap,\neg A_aq}$};
\path (1b) edge [-latex, loop, looseness=5,in=105,out=70] node [right] {$\s b$} (1b);

\path (1b) edge [-latex] (1.south) node [right, xshift=-22pt, yshift=18pt] {$\s a$};

\node[above=of !, yshift=-34pt] (anchor1) {};
\node [deepsquare,fit={(!)(1b)(anchor1)}
](squareb) {};
\node[fill=white] (square name 1) at (squareb.south) [yshift=-2pt]{$\s A_b B_a \{p,q\}, A_b \neg B_a \{p,q\}$};

\path (!) edge [-latex] (1b) node [above, xshift=34pt, yshift=-0pt] {$\s b$};

\node[below=of !, yshift=25pt] (anchor1) {};
\node[left=of !, xshift=34pt] (anchor2) {};
\node[right=of 1b, xshift=-34pt] (anchor3) {};
\node [deepsquare,fit={($(!) +(4,0mm)$)(1)($(1b) +(-40mm,0)$)(anchor1)(anchor2)(anchor3)}](squareall) {};
\node[fill=white] (square name 1) at (squareall.north) [yshift=1pt]{$\s A_b p, A_bq$};

\end{tikzpicture}
\end{center}

\caption{Attention model $(\mathcal{M}'\otimes\mathcal{R}(\{B_ap, \neg B_aq\}),(w',e))$ for $\mathcal{L}_{\text{GA}}$, where $(\mathcal{M}',w')$ is given in Example \ref{ex: updated general attention}. 
We use the same conventions as before and additionally we let $B_a \{p,q\} := B_ap\wedge B_aq$. While earlier $A_ap$ occurring at $w$ represented that $A_a p \in V(w)$, here it represents that $p \in \mathcal{A}_a(w)$ (similarly for $\neg A_a p$).  
}\label{figure: updated general attention}
\end{figure} 
\end{example}

This example illustrates how the logic of general attention may be used, for example by an AI agent, to reason about and discover the attentional biases of other agents. While the example is relatively basic, the underlying logical framework is very general, and can handle much more complex scenarios. 

\section*{Concluding Remarks}
In this paper, we first proposed a generalization of edge-conditioned event models that we used to model propositional attention. We showed that these models can be exponentially more succinct than standard event model, and never more than linearly larger. Then, we adopted them to capture attention to arbitrary formulas. In future work, we would like to further investigate the properties of our general attention framework and extend it to incorporate additional features of attention, such as capacity constraints. Limiting the attention resources of agents prompts questions about exactly which formulas the agent will learn. For example, consider an agent that has the capacity to attend to $n$ formulas at most. If $\Gamma$ reveals more than $n$ formulas, which formulas will the agent prioritize? To address this issue, we need an ordering of the formulas that determines which ones are attended first and thus learned. We will explore this topic in our next paper.

A standard limitation of classical DEL, inherited by our framework, is that agents cannot recover from false beliefs: If an agent believes $\neg p$ and a public announcement of $p$ occurs, the agent will come to believe every formula. The standard response is that these cases require to move to a richer framework, such as DEL based on plausibility models~\cite{baltag2008qualitative}. Adapting our logic of general attention to that framework is another interesting direction for future work.



	\section*{Acknowledgments} 
Gaia Belardinelli is funded by Independent Research Fund Denmark (grant no.\ 4255-00020B).
Thomas Bolander is supported by the Independent Research Fund Denmark (grant no.\ 10.46540/4258-00060B). Sebastian Watzl is funded by the European Union's Horizon 2020 research and innovation program under grant agreement no. 101003208 (ERC Consolidator Grant 2020).

\appendix 
\section*{APPENDIX}

In this Technical Appendix, we prove the theorems stated in the main text.

\section{Basic Definitions}

In this section, we provide the detailed definitions of the standard and well-known concepts from the literature that were mentioned or required in the main text, but where a detailed definition was not provided.  
\label{sec:basic_definitions}
\subsection{Isomorphisms and bisimulations}


The following definitions of isomorphism and bisimulation are standard in modal logic
\cite{blac.ea:moda}. 


\begin{definition}[Isomorphism]
	\label{def:isomorphism}
    Let $\mathcal{M}=(W,R,V)$ be a Kripke model for $\mathcal{L}$ and $\mathcal{M}'=(W',R',V')$ a Kripke model for $\mathcal{L}'$ where $At(\mathcal{L}) = At(\mathcal{L}')$. 
	An \emph{isomorphism} between $\mathcal{M}$ and  $\mathcal{M}'$ 
	is a bijection $f:W \rightarrow W'$ such that \begin{enumerate}
	\item $p\in V(w)$ iff $p\in V'(f(w))$, for all $w\in W$ and $p\in At(\mathcal{L})$;
	\item $(w,v)\in R_a$ iff $(f(w),f(v))\in R_a'$, for all $w,v\in W$ and $a\in Ag$.
\end{enumerate}
When an isomorphism exists between two models, we say that they are \emph{isomorphic}.
\end{definition}
\begin{definition}[Bisimulation]\label{def: bisimulation}
	Let $\mathcal{M}=(W,R,V)$ be a Kripke model for $\mathcal{L}$ and $\mathcal{M}'=(W',R',V')$ a Kripke model for $\mathcal{L}'$ where $At(\mathcal{L}) = At(\mathcal{L}')$. A \emph{bisimulation} between $\mathcal{M}$ and $\mathcal{M}'$ 
    is a non-empty binary relation $Z\subseteq W \times W'$ such that for all $(w,w')\in Z$ and $a\in Ag$:

\noindent \begin{enumerate}
	\item{[Atom]}: $p\in V(w)$ iff $p\in V'(w')$, for all $p\in At(\mathcal{L})$;
\item{[Forth]}: If $(w,v)\in R_a$ then there exists $v'\in W'$ such that $(w',v')\in R'_a$ and $(v,v')\in Z$;
\item{[Back]}: If $(w',v')\in R'_a$ then there exists $v\in W$ such that $(w,v)\in R_a$ and $(v,v')\in Z$;
\end{enumerate}
A bisimulation between pointed Kripke models $(\mathcal{M},w)$ and $(\mathcal{M}',w')$ further has to satisfy that $(w,w')\in Z$. When a bisimulation exists between two models, we say that they are \emph{bisimilar}. 
\end{definition} 	

In the main text, we only defined update equivalence for non-pointed models. However, that definition immediately generalizes to (multi-)pointed models:
\begin{definition}[Update equivalence~\cite{kooi2011generalized,eijck2012action}] 
	\label{def:update equivalence}
	Let $\mathcal{D}$ be a standard or edge-conditioned event model for $\mathcal{L}$ 
	and $\mathcal{D}'$ a standard or edge-conditioned event model for $\mathcal{L}'$, where $At(\mathcal{L}) = At(\mathcal{L}')$.
	We say that $\mathcal{D}$ is \emph{update equivalent to} $\mathcal{D}'$ if for all Kripke models $\mathcal{M}$ with atom set $At(\mathcal{L})$,
	$\mathcal{M}\otimes\mathcal{E}$ and $\mathcal{M}\otimes\mathcal{E}'$ are bisimilar. 
	
	Where $\mathcal{D}$ and $\mathcal{D}'$ are (multi-)pointed event models, $\mathcal{D}$ is update equivalent to $\mathcal{D}'$, if for all pointed Kripke models $(\mathcal{M},w)$ with atom set $At(\mathcal{L})$:
\begin{itemize}
	\item$\mathcal{D}$ is applicable in $(\mathcal{M},w)$ iff $\mathcal{D}'$ is applicable in $(\mathcal{M},w)$,
\item $(\mathcal{M},w) \otimes \mathcal{D}$ and $(\mathcal{M},w) \otimes \mathcal{D}’$ are bisimilar.
\end{itemize}
\end{definition}

\subsection{Generalized arrow updates}
The definitions of generalized arrow updates below are from the work of Kooi and Rennes~[\citeyear{kooi2011generalized}]. We have adapted the original formulations slightly to fit our notational conventions. 
\begin{definition}[Generalized arrow update%
	]
	A \emph{generalized arrow update} $\mathcal{U}$ for $\mathcal{L}$ is a pair $(O, \mathsf{a})$ consisting of a finite nonempty set $O$ of outcomes and an arrow function $\mathsf{a}:\mathcal{A}\times O\rightarrow \mathcal{L} \times O\times \mathcal{L}$, with notation $\mathsf{a}_a(o)=\mathsf{a}(a,o)$. The tuple $(\varphi, o',\varphi')\in \mathsf{a}_a(o)$ is an \emph{$a$-arrow} with  \emph{source condition} $\varphi$, \emph{target} $o'$, and \emph{target condition} $\varphi'$. The set of \emph{generalized arrow updates} for $\mathcal{L}$ is denoted $\gau(\mathcal{L})$. 
	Where  $\mathcal{U}$ is a generalized arrow update and $o\in O$ is the \emph{actual outcome}, we call $(\mathcal{U},o)$ a \emph{pointed generalized arrow update}. 
\end{definition}

\begin{definition}[Product update with generalized arrow updates%
	]
	Let $\mathcal{M}=(W,R,V)$ be a Kripke model and $\mathcal{U}=(O,\mathsf{a})$ a generalized arrow update, both 
	for the same language $\mathcal{L}$. 
	The \emph{(product) update} of $\mathcal{M}$ with $\mathcal{U}$ is the Kripke model $\mathcal{M}\otimes\mathcal{U}=(W',R',V')$ where

	
	\smallskip 
	$W'=W\times O$
	
	\vspace{-5mm}
	\begin{multline*}
	R’_a=\{((w,o),(w’,o’))\in W’\times W’\colon (w,w’)\in R_a, \\ (\varphi , o’, \varphi’)\in \mathsf{a}_a(o),  (\mathcal{M},w)\vDash \varphi, (\mathcal{M},w’)\vDash \varphi’\},
	\end{multline*}

	$V'(w,o)= \{ p \in At(\mathcal{L})\colon p\in V(w)\}$.\footnote{For any function $f$ that takes a pair $(x,y)$ as argument, we allow writing $f(x,y)$ for $f((x,y))$, so $V'(w,o)$ is short for $V'((w,o))$.}

	\smallskip
	The product update of a pointed Kripke model $(\mathcal{M},w)$ with a pointed generalized arrow update $(\mathcal{U},o)$ is the pointed Kripke model $(\mathcal{M}\otimes\mathcal{U},(w,o))$.
\end{definition}
We define the \emph{language of generalized arrow updates} $\mathcal{L}_\text{GAU}$ as the language given by the grammar of $\mathcal{L}_\text{EL}$ extended with the clause $\varphi ::= [\mathcal{U}] \varphi$, where $\mathcal{U}$ is a pointed generalized arrow update. The semantics is defined as for $\mathcal{L}_\text{DEL}$ (Def.~\ref{def:satisfaction}), except that the product update in the semantics of $[\mathcal{U}] \varphi$ uses the generalized arrow update definition above, and we take $\mathcal{U}$ to be applicable in \emph{any} pointed Kripke model $(\mathcal{M},w)$.

\subsection{Sizes of formulas and models} 
\label{section:sizes} 
We define the sizes of formulas and event models. These are fairly standard definitions, see e.g. \cite{kooi2011generalized,bolander2020del,bolander2023parameterized}. Below we generalize to edge-conditioned event models. 

		
		The \emph{length} (or \emph{size}) of a formula $\varphi$, notation $|\varphi|$, is defined as the length of its representation as a string (so the number of symbols it consists of). 
		The \emph{size} of a set $S$, notation $|S|$, is its cardinality. The \emph{size of a standard event model} $\mathcal{E}=(E,Q,pre)$, notation $| \mathcal{E} |$, is defined by
	$$|\mathcal{E}| = |E| + \sum_{a \in Ag} | Q_a | + \sum_{e \in E} |pre(e)|$$
	%

	The \emph{size of an edge-conditioned event model} $\mathcal{C}=(E,Q,pre)$, notation $| \mathcal{C} |$, is defined by
	$$|\mathcal{C}| = | E| +   \sum_{a \in Ag} \bigl(|Q_a|\ + \!\!\!\!\sum_{(e:\varphi, f:\psi)\in Q_a} (|\varphi| + |\psi|) \bigr) + \sum_{e \in E} |pre(e)|$$	
	The \emph{size of a generalized arrow update} $\mathcal{U}=(O,\mathsf{a})$, notation $| \mathcal{U} |$ is defined by 
	$$|\mathcal{U}| = | O| +   \sum_{a \in Ag}\sum_{o\in O} (|\mathsf{a}_a(o)|+\!\!\!\!\sum_{(\varphi, o,\psi)\in \mathsf{a}_a(o)} (|\varphi| + |\psi|))$$
	
   For pointed event models, we need to also include the size of pointing out the actual events. However, this is just one bit per event, so it can be incorporated by replacing $|E|$ with $2|E|$ above, a constant factor. It can hence be ignored in the asymptotic analyses, and we will for simplicity do so in the following.


\subsection{Language hierarchies of DEL} 
\label{section:language hierarchies}
We now define language hierarchies of DEL, similarly to~\citeauthor{kooi2011generalized}~[\citeyear{kooi2011generalized}]. We call the modalities $[\mathcal{E}]$ of $\mathcal{L}_\text{DEL}$, $[\mathcal{C}]$ of $\mathcal{L}_\text{ECM}$ and $[\mathcal{U}]$ of $\mathcal{L}_\text{GAU}$ the \emph{dynamic modalities}. When we want to specify a dynamic modality that could be any of these, we use the notation $[\mathcal{D}]$, i.e.\ we use $\mathcal{D}$ as a placeholder for $\mathcal{E}$, $\mathcal{C}$ or $\mathcal{U}$.
For a standard event model $\mathcal{E} = (E,Q,pre)$, its set of \emph{conditions}, denoted $cond(\mathcal{E})$, is given by $cond(\mathcal{E}) = \{pre(e): e \in E \}$. For an edge-conditioned event model $\mathcal{C} = (E,Q,pre)$ the corresponding set also includes the source and target conditions, that is, $cond(\mathcal{C}) = \{ pre(e): e \in E \} \cup \{ \varphi: (e{:}\varphi,f{:}\psi) \in Q_a \} \cup \{ \psi: (e{:}\varphi,f{:}\psi) \in Q_a \}$. For a generalized arrow update $\mathcal{U} = (O,\mathsf{a})$, the set of conditions is $cond(\mathcal{U}) = \{ \varphi: (\varphi,o,\psi) \in \mathsf{a}_a(o) \} \cup \{ \psi: (\varphi,o,\psi) \in \mathsf{a}_a(o) \}$. 

Let $\mathcal{L}$ be one of the languages $\mathcal{L}_\text{DEL}$, $\mathcal{L}_\text{ECM}$ or $\mathcal{L}_\text{GAU}$. 
We let $\mathcal{L}^0$ denote the language $\mathcal{L}$ without the dynamic modality, i.e.\ without the modal formulas $[\mathcal{D}] \varphi$. Formulas of $\mathcal{L}^0$ are called \emph{static formulas}. We now inductively define languages $\mathcal{L}^i$, $i > 0$, as follows. The language $\mathcal{L}^{i}$ consists of all the formulas $\varphi \in \mathcal{L}$ satisfying that for any subformula $[\mathcal{D}] \psi$ of $\varphi$, $cond(\mathcal{D}) \subseteq \mathcal{L}^{i-1}$. For instance, we have that $\mathcal{L}^{1}_\text{DEL}$ includes the dynamic modality $[\mathcal{E}]\varphi$, but only in cases when the preconditions in $\mathcal{E}$ are static formulas. We have $\mathcal{L} = \cup_{i \geq 0} \mathcal{L}^i$ due to the recursive construction of the language $\mathcal{L}$.

When $\varphi \in \mathcal{L}^0$, we say that $\varphi$ is of \emph{level $0$} in the language hierarchy.   
For $i > 0$, when $\varphi \in \mathcal{L}^i \setminus \mathcal{L}^{i-1}$, we say that $\varphi$ is of \emph{level} $i$ in the language hierarchy. Similarly, we say that an event model or generalized arrow update $\mathcal{D}$ is of \emph{level} $i$ in the language hierarchy when $cond(\mathcal{D}) \subseteq \mathcal{L}^i$ and 
$cond(\mathcal{D}) \not\subseteq \mathcal{L}^{i-1}$.

\section{Soundness and Completeness of DEL with Edge-Conditions}

	\begin{proof}[\textbf{Proof of Theorem 
		\ref{thm: soundness and completeness}}] (\emph{The axiomatization in Table \ref{tab:logic} is sound and complete.})
		Completeness follows by standard reduction arguments \cite{ditmarsch2007dynamic}, realising that we have reduction axioms for the dynamic modality applied to any of the other types of formulas. We show the soundness of the axiom for belief update. 
		$$[(\mathcal{C},e)]B_a\psi \leftrightarrow \bigl(pre(e)\rightarrow\hspace{6pt} \bigwedge_{\mathclap{(e{:}\chi,f{:}\chi')\in Q_a}}\hspace{6pt}(\chi\rightarrow B_a(\chi' \rightarrow [(\mathcal{C},f)]\psi))\bigr)$$
		
		Let $(\mathcal{M},w)=((W,R,V),w)$ be a pointed Kripke model for $\mathcal{L}_{\text{ECM}}$ and let $(\mathcal{C},e)=((E,Q,pre),e)$ be a pointed edge-conditioned model for $\mathcal{L}_{\text{ECM}}$.
		
		($\Rightarrow$) 
		We want to show that by assuming $(\mathcal{M},w)\vDash [(\mathcal{C},e)]B_a\psi$, it follows that $(\mathcal{M},w)\vDash pre(e)\rightarrow \bigwedge_{(e{:}\chi,f{:}\chi')\in Q_a}(\chi\rightarrow B_a(\chi' \rightarrow [(\mathcal{C},f)]\psi))$. So assume that $(\mathcal{M},w)\vDash [(\mathcal{C},e)]B_a\psi$ and that $(\mathcal{M},w)\vDash pre(e)$. 
		Then 
		$(\mathcal{M}\otimes \mathcal{C},(w,e))=((W',R',V'), (w,e))$ is the edge-conditioned product update of $(\mathcal{M},w)$ and $(\mathcal{C},e)$.
		By $(\mathcal{M},w)\vDash [(\mathcal{C},e)]B_a\psi$ and $(\mathcal{M},w)\vDash pre(e)$ it follows that $(\mathcal{M}\otimes \mathcal{C},(w,e))\vDash B_a\psi$, by semantics of the dynamic modality. Then $(\mathcal{M}\otimes \mathcal{C},(u,g))\vDash \psi$ for all $((w,e),(u,g))\in R'_a$, by semantics of the belief modality.

		As we need to show that $(\mathcal{M},w)\vDash\bigwedge_{(e{:}\chi,f{:}\chi')\in Q_a}(\chi\rightarrow B_a(\chi' \rightarrow [(\mathcal{C},f)]\psi))$, consider an edge-condition $(e{:}\chi,f{:}\chi')\in Q_a$ with arbitrary $f,\chi,\chi'$, and assume that $(\mathcal{M},w)\vDash \chi$. As we need to show that $(\mathcal{M},w)\vDash B_a(\chi' \rightarrow [(\mathcal{C},f)]\psi))$, consider an arbitrary $(w,v)\in R_a$ and assume that $(\mathcal{M},v)\vDash  \chi'$. To prove the desired, it remains to show that $(\mathcal{M},v)\vDash[(\mathcal{C},f)]\psi$.

		Now we have two cases: either $(\mathcal{M},v)\vDash pre(f)$ or not.
		If $(\mathcal{M},v)\not\vDash pre(f)$, then it trivially follows that $(\mathcal{M},v)\vDash [(\mathcal{C},f)]\psi$, by semantics of the dynamic modality. If instead $(\mathcal{M},v)\vDash pre(f)$, then, by definition of edge-conditioned product update, $(v,f)\in W'$. As we earlier assumed that $(e{:}\chi,f{:}\chi')\in Q_a$, $(\mathcal{M},w)\vDash \chi$, $(\mathcal{M},v)\vDash  \chi'$, and $(w,v)\in R_a$, then, by definition of edge-conditioned product update, we have that $((w,e),(v,f))\in R'_a$. Now recall that $(\mathcal{M}\otimes \mathcal{C},(u,g))\vDash \psi$ for all $((w,e),(u,g))\in R'_a$. Then, also for $(v,f)$ it holds that $(\mathcal{M}\otimes \mathcal{C},(v,f))\vDash \psi$, and, by semantics of dynamic modality, $(\mathcal{M},v)\vDash [(\mathcal{C},f)]\psi$. Hence, in both cases, $(\mathcal{M},v)\vDash [(\mathcal{C},f)]\psi$, as required.
		

		($\Leftarrow$) We want to show that by assuming that $(\mathcal{M},w)\vDash pre(e)\rightarrow \bigwedge_{(e{:}\chi,f{:}\chi')\in Q_a}(\chi \rightarrow B_a(\chi'\rightarrow [(\mathcal{C},f)]\psi))$, it follows that $(\mathcal{M},w)\vDash [(\mathcal{C},e)]B_a\psi$. So let $(\mathcal{M},w)\vDash pre(e)\rightarrow \bigwedge_{(e{:}\chi,f{:}\chi')\in Q_a}(\chi \rightarrow B_a(\chi'\rightarrow [(\mathcal{C},f)]\psi))$. Now we have two cases: either  $(\mathcal{M},w)\vDash pre(e)$ or not. We want to show that in both cases we can derive that $(\mathcal{M},w)\vDash [(\mathcal{C},e)]B_a\psi$. If $(\mathcal{M},w)\not\vDash pre(e)$, then trivially $(\mathcal{M},w)\vDash [(\mathcal{C},e)]B_a\psi$. So suppose that $(\mathcal{M},w)\vDash pre(e)$. By our initial assumption and modus ponens, it follows that $(\mathcal{M},w)\vDash \bigwedge_{(e{:}\chi,f{:}\chi')\in Q_a}(\chi \rightarrow B_a(\chi'\rightarrow [(\mathcal{C},f)]\psi))$, that is, $(\mathcal{M},w)\vDash \chi \rightarrow B_a(\chi'\rightarrow [(\mathcal{C},f)]\psi)$ for all $(e{:}\chi,f{:}\chi')\in Q_a$. As we want to show that $(\mathcal{M},w)\vDash [(\mathcal{C},e)]B_a\psi$, consider the model $(\mathcal{M}\otimes\mathcal{C},(w,e))=((W',R',V'),(w,e))$ and take an arbitrary $((w,e),(v,f))\in R'_a$. We need to show that $(\mathcal{M}\otimes\mathcal{C},(v,f))\vDash \psi$.
		By definition of edge-conditioned product update, $((w,e),(v,f))\in R'_a$ implies that $(w,v)\in R_a$ and that there exists an edge-condition $(e{:}\chi,f{:}\chi')\in Q_a$ with $(\mathcal{M},w)\vDash \chi$ and $(\mathcal{M},v)\vDash \chi'$. Then, by $(\mathcal{M},w)\vDash \chi$, $(\mathcal{M},w)\vDash \chi \rightarrow B_a(\chi'\rightarrow [(\mathcal{C},f)]\psi)$ and modus ponens, we know that $(\mathcal{M},w)\vDash  B_a(\chi'\rightarrow [(\mathcal{C},f)]\psi)$, that is, for all $(w,v')\in R_a$, $(\mathcal{M},v')\vDash  \chi'\rightarrow [(\mathcal{C},f)]\psi$. As $(w,v)\in R_a$, it follows that $(\mathcal{M},v)\vDash  \chi'\rightarrow [(\mathcal{C},f)]\psi$, and since $(\mathcal{M},v)\vDash \chi'$, it follows that $(\mathcal{M},v)\vDash [(\mathcal{C},f)]\psi$. Recall that $((w,e),(v,f))\in R'_a$, so $(v,f)\in W'$ and  $(\mathcal{M},v)\vDash pre(f)$, by definition of product update. Then, by semantics of the dynamic modality, $(\mathcal{M},v)\vDash [(\mathcal{C},f)]\psi$ implies that  $(\mathcal{M}\otimes\mathcal{C},(v,f))\vDash \psi$, as required. 
		%
	\end{proof}
	
	\section{Comparisons of Edge-Conditioned Event Models and Standard Event Models}


\begin{proof}[\textbf{Proof of Theorem~\ref{prop:F-H-equivalence}}](\emph{For any conjunction of propositional literals $\varphi$, 
		$\mathcal{F}(\varphi)$ and $\mathcal{H}(\varphi)$ are update equivalent.})
	Let $\varphi=\ell(p_1)\wedge \dots \wedge \ell(p_n)$. Let $\mathcal{H}(\varphi)=((E,Q,pre), \varphi)$ be the edge-conditioned event model for propositional attention representing the revelation of $\varphi$, and $\mathcal{F}(\varphi)=((E',Q',pre'),E_d)$ be the standard event model for propositional attention representing the same revelation. 
	Note that $\mathcal{H}(\varphi)$ is an edge-conditioned event model for $\mathcal{L}_{\text{PA}^+}$ and $\mathcal{F}(\varphi)$ is a standard event model for $\mathcal{L}_{\text{PA}}$, and that $At(\mathcal{L}_{\text{PA}^+})=At(\mathcal{L}_{\text{PA}})$. Let $(\mathcal{M},w)=((W,R,V),w)$ be any pointed Kripke model with atom set $At(\mathcal{L}_\text{PA})$. 
	
As $\mathcal{H}(\varphi)$ and $\mathcal{F}(\varphi)$ are pointed and multi-pointed event models, respectively, Definition \ref{def:update equivalence} of update equivalence for multi-pointed models tells us that to show that $\mathcal{H}(\varphi)$ and $\mathcal{F}(\varphi)$ are update equivalent,
	we need to show that (1) $\mathcal{H}(\varphi)$ is applicable in $(\mathcal{M},w)$ iff $\mathcal{F}(\varphi)$ is, and (2) $(\mathcal{M},w)\otimes \mathcal{H}(\varphi)$ and $(\mathcal{M},w)\otimes \mathcal{F}(\varphi)$ are bisimilar. We start by showing (1). Notice that the actual event in $\mathcal{H}(\varphi)$ is $\varphi$ and all actual events in $\mathcal{F}(\varphi)$ contain $\varphi$. Then, we have two cases: either $(\mathcal{M},w)\vDash \varphi$ or not. If not, then both event models are not applicable in $(\mathcal{M},w)$, and the desired holds. Consider now the case in which $(\mathcal{M},w)\vDash \varphi$. Then $\mathcal{H}(\varphi)$ is applicable in $(\mathcal{M},w)$, and we only need to show that also $\mathcal{F}(\varphi)$ is. As for all $a\in Ag$, there is a unique $T_a\subseteq At(\varphi)$ such that $(\mathcal{M},w)\models \bigwedge_{a\in Ag}(\bigwedge_{p_i\in T_a}A_a p_i\wedge \bigwedge_{p_i\in At(\varphi)\setminus T_a} \neg A_a p_i)$, then there is a unique event $e \in E_d$ containing $\varphi$ and $\bigwedge_{a\in Ag}(\bigwedge_{p_i\in T_a}A_a p_i\wedge \bigwedge_{p_i\in At(\varphi)\setminus T_a} \neg A_a p_i)$ that is such that $(\mathcal{M},w)\vDash e$. By definition of applicability of multi-pointed event models, it follows that $\mathcal{F}(\varphi)$ is applicable in $(\mathcal{M},w)$, as required.
	
	We now show (2), namely that $(\mathcal{M},w)\otimes \mathcal{H}(\varphi)$ and $(\mathcal{M},w)\otimes \mathcal{F}(\varphi)$ are bisimilar.  
Notationally, let $(\mathcal{M},w)\otimes \mathcal{H}(\varphi)=
	((W',R',V'),(w,\varphi))$ and $(\mathcal{M},w)\otimes \mathcal{F}(\varphi)=
	((W'',R'',V''),(w,e))$. Additionally, for a set $X\subseteq At(\varphi)$ and agent $a\in Ag$, we define the following abbreviations: 
	\begin{itemize}
		\item $X:=\bigwedge_{p_i\in X} p_i$; 
		\item $\ell(X):=\bigwedge_{p_i\in X}\ell(p_i)$, where $\ell(p_i)$ is the literal from $\varphi$ built on $p_i$;
		\item $A_a X:=\bigwedge_{p_i\in X} A_a p_i$;
	\end{itemize}
	Recall that for any conjunction of literals $\psi = \bigwedge_{1 \leq i \leq n} \ell_i$ and any literal $\ell$, if $\ell = \ell_i$ for some $i$ we write $\ell \in \psi$. More generally, if $\chi$ is a subformula of $\psi$ we write $\chi \in \psi$.
	
	Consider the relation $Z\subseteq W'\times W''$, defined by $Z=$
	\begin{multline*}
		\hspace{-10pt}	
		\{((w,\bigwedge_{\mathclap{p_i\in S}}\ell (p_i) ),(w,\bigwedge_{\mathclap{p_i\in S}} \ell(p_i) \wedge \hspace{2pt}\bigwedge_{\mathclap{a\in Ag}\hspace{2pt}}\hspace{2pt} (\hspace{2pt}\bigwedge_{\hspace{2pt}\mathclap{p_i\in T_a}}A_a p_i \hspace{2pt}\wedge\hspace{2pt} \bigwedge_{\mathclap{p_i\in S\setminus T_a}}\neg A_a p_i)))\colon \\S\subseteq At(\varphi) \text{ and for all } a\in Ag, T_a\subseteq S\}
	\end{multline*}
	
	Using the abbreviations, the set can be written as:
	\begin{multline*}
		\hspace{-10pt}		Z=\{((w, \ell(S) ),(w, \ell(S) \wedge \bigwedge_{a\in Ag} (A_a T_a \wedge \neg A_a (S\setminus T_a)))\colon \\S\subseteq At(\varphi) \text{ and for all } a\in Ag, T_a\subseteq S\}
	\end{multline*}
	We first show that $Z$ is a bisimulation, i.e. that it satisfies the three requirements of bisimulations (see Definition \ref{def: bisimulation}). After that, we show that $((w,\varphi),(w,e))\in Z$.
	
	Let $S\subseteq At(\varphi)$ and, for all $a\in Ag$, let $T_a\subseteq S$. Consider $((w,\ell(S)),(w,\ell(S)\wedge \bigwedge_{a\in Ag}(A_a T_a\wedge \neg A_aS\setminus T_a)))\in Z$.
	
	[Atom]: Let $p\in At(\mathcal{L}_\text{PA})$. Then $p\in V'((w,\ell(S)))$ iff (by def. of product update) $p\in V''((w,\ell(S)\wedge \bigwedge_{a\in Ag}(A_a T_a\wedge \neg A_a S\setminus T_a)))$, which is the required.
	
	[Forth]: Let $b\in Ag$ and $((w,\ell(S)),(v,\ell(S')))\in R'_b$. We want to show that there exists a world $(v',\ell(S'')\wedge \bigwedge_{a\in Ag}(A_a T_a'\wedge \neg A_a S'' \setminus T_a'))\in W''$, such that: \begin{enumerate}
		\item[(1)] $((w,\ell(S)\wedge \bigwedge_{a\in Ag}(A_a T_a\wedge \neg A_aS\setminus T_a),(v',\ell(S'')\wedge \bigwedge_{a\in Ag}(A_a T_a'\wedge \neg A_aS''\setminus T_a')))\in R''_b$; 
		\item[(2)] $((v,\ell(S')),(v',\ell(S'')\wedge \bigwedge_{a\in Ag}(A_a T_a'\wedge \neg A_aS''\setminus T_a')))\in Z$. 
	\end{enumerate}

	As $((w,\ell(S)),(v,\ell(S')))\in R'_b$, then $(v,\ell(S'))\in W'$, and by definition of product update $\ell(S')\in E$, with $S'\subseteq At(\varphi)$ by definition of edge-conditioned event model for propositional attention. As $S'\subseteq At(\varphi)$, then the standard event model for propositional attention $\mathcal{F}(\varphi)$ must contain, for all $T'_a\subseteq S'$, $a\in Ag$, an event $\ell(S')\wedge \bigwedge_{a\in Ag}(A_a T'_a\wedge \neg A_aS'\setminus T'_a)$.
	So define $X_a=\{q\in S' \colon (\mathcal{M},v)\vDash A_a q\}$, for all $a\in Ag$. Since $X_a\subseteq S'$, for all $a\in Ag$, then there is a unique event $f$ in $E'$ such that $$f=\ell(S')\wedge \bigwedge_{a\in Ag}(A_a X_a\wedge \neg A_aS'\setminus X_a).$$ 
  As $(v,\ell(S'))\in W'$, then by product update definition, $(\mathcal{M},v)\vDash \ell(S')$. Hence, $(\mathcal{M},v)\vDash f$ by construction of $f$, and since $pre(f)=f$, then $(v,f)\in W''$ by definition of product update. 
	
	We now want to show that $(w,\ell(S)\wedge \bigwedge_{a\in Ag}(A_a T_a\wedge \neg A_aS\setminus T_a)),(v,f))\in R''_b$, i.e. property (1).
As $((w,\ell(S)),(v,\ell(S')))\in R'_b$, then by def. of edge-conditioned product update, $(w,v)\in R_b$. Then, to reach the desired result, we only need to show that $(\bigwedge \ell(S) \bigwedge_{a\in Ag}(A_a T_a\wedge \neg A_aS\setminus T_a),f)\in Q'_b$, as then by definition of standard product update, we would have that $(w,\ell(S)\wedge \bigwedge_{a\in Ag}(A_a T_a\wedge \neg A_aS\setminus T_a)),(v,f))\in R''_b$.
	
	By $((w,\ell(S)),(v,\ell(S')))\in R'_b$ and definition of edge-conditioned product update, there exists $\chi,\chi'\in \mathcal{L}_\text{PA}$ such that $(\ell(S){:}\chi,\ell(S'){:}\chi')\in Q_b$. By definition of the edge-conditioned event models for propositional attention, we must have that $\chi=A_b S' \wedge \neg A_b (S\setminus S')$ and $\chi'=A_b S'$, i.e., 
$(\bigwedge S{:}A_b S' \wedge \neg A_b (S\setminus S'), \bigwedge S'{:}A_b S')\in Q_b$ with $(\mathcal{M},w)\vDash A_b S' \wedge \neg A_b (S\setminus S'))$ and $(\mathcal{M},v)\vDash A_b S'$. 
	
	As $(w,\ell(S)\wedge \bigwedge_{a\in Ag}(A_a T_a\wedge \neg A_a(S\setminus T_a)))\in W''$ then by definition of product update $(\mathcal{M},w)\vDash \ell(S)\wedge \bigwedge_{a\in Ag}(A_a T_a\wedge \neg A_aS\setminus T_a)$. As we just saw that $(\mathcal{M},w)\vDash A_b S' \wedge \neg A_b (S\setminus S')$, it must be the case that $T_b = S'$, which implies that $A_b S' \wedge \neg A_b (S\setminus S')$ is a subformula of $\ell(S)\wedge \bigwedge_{a\in Ag}(A_a T_a\wedge \neg A_a(S\setminus T_a))$, i.e. both $A_b S'$ and $\neg A_b (S\setminus S')$ are its subformulas. 

	By definition of product update, $(w,\ell(S)\wedge \bigwedge_{a\in Ag}(A_a T_a\wedge \neg A_a(S\setminus T_a)))\in W''$ furthermore implies that $\ell(S)\wedge \bigwedge_{a\in Ag}(A_a T_a\wedge \neg A_a(S\setminus T_a))\in E'$. 
	We now want to know under which conditions, for an event $g\in E'$, we can imply that $(\ell (S)\wedge \bigwedge_{a\in Ag}(A_a T_a\wedge \neg A_a(S\setminus T_a)),g)\in Q'_b$. Since $A_b S'\in \ell(S)\wedge \bigwedge_{a\in Ag}(A_a T_a\wedge \neg A_a(S\setminus T_a))$, then by \textsc{Attentiveness}, we must have
$\ell(S')\in g$ and $A_b S' \in g$, and since $\neg A_b (S\setminus S')\in \bigwedge S\wedge \bigwedge_{a\in Ag}(A_a T_a\wedge \neg A_a(S\setminus T_a))$, then by \textsc{Inertia} we must have that $\ell(S\setminus S')\not\in g$. Moreover, 
since $A_bAt(\varphi)\setminus S\not\in \ell(S)\wedge \bigwedge_{a\in Ag}(A_a T_a\wedge \neg A_a(S\setminus T_a))$, then $\ell(At(\varphi)\setminus S)\not\in g$. Hence, if $\ell(S')\in g, A_b S' \in g$ and $\ell(S\setminus S')\not\in g, \ell(At(\varphi)\setminus S)\not\in g$ then $(\ell (S)\wedge \bigwedge_{a\in Ag}(A_a T_a\wedge \neg A_a(S\setminus T_a)),g)\in Q'_b$. Observe that this is the case for $f$: Recall that  $(\mathcal{M},v)\vDash A_b S'$, and so by $(v,f)\in W'$ we know that $A_b S'\in f$. Moreover, by construction of $f$, $\ell(S')\in f , \ell(S\setminus S')\not\in f$ and $\ell(At(\varphi)\setminus S)\not\in f$. Hence, $(\ell (S)\wedge \bigwedge_{a\in Ag}(A_a T_a\wedge \neg A_a(S\setminus T_a)),f)\in Q'_b$, as required.

	It remains to show that $((v,\ell(S')),(v,f))\in Z$, i.e.\ property (2).
	Notice that to show that the pair $((v,\ell(S')),(v,\ell(S')\wedge \bigwedge_{a\in Ag}(A_a X_a\wedge \neg A_aS'\setminus X_a)))$ belongs to $Z$, we only need to show that $S'\subseteq At(\varphi)$ and $X_a\subseteq S'$, for all $a\in Ag$. It holds that $S'\subseteq At(\varphi)$ as $(w,\ell(S'))\in W'$ and so $\ell(S')\in E$, by definition of product update, and $S'\subseteq At(\varphi)$ by definition of edge-conditioned event model for propositional attention. It holds that $X_a\subseteq S'$, for all $a \in Ag$ by construction of $f$.
	
	[Back]: Let $b\in Ag$ and $((w,\ell (S)\wedge \bigwedge_{a\in Ag}(A_a T_a\wedge \neg A_a(S\setminus T_a))),(v,\ell (S')\wedge \bigwedge_{a\in Ag}(A_a T'_a\wedge \neg A_a(S'\setminus T'_a))))\in R''_b$. We want to show that there exists a world $(v',\ell(S''))\in W'$, such that: \begin{enumerate}
		\item[(1b)] $((w,\ell(S)),(v',\ell(S'')))\in R'_b$; 
		\item[(2b)] $((v',\ell(S'')),(v,\ell(S')\wedge \bigwedge_{a\in Ag}(A_a T_a'\wedge \neg A_aS'\setminus T_a')))\in Z$. 
	\end{enumerate}
	
	As $((w,\ell (S)\wedge \bigwedge_{a\in Ag}(A_a T_a\wedge \neg A_a(S\setminus T_a))),(v,\ell (S')\wedge \bigwedge_{a\in Ag}(A_a T'_a\wedge \neg A_a(S'\setminus T'_a))))\in R''_b$ then $(v,\ell (S')\wedge \bigwedge_{a\in Ag}(A_a T'_a\wedge \neg A_a(S'\setminus T'_a))\in W''$ and by definition of product update $\ell (S')\wedge \bigwedge_{a\in Ag}(A_a T'_a\wedge \neg A_a(S'\setminus T'_a)\in E'$, with $S'\subseteq At(\varphi)$, by def. of standard event model for propositional attention. Since $S'\subseteq At(\varphi)$, then the edge-conditioned event model for propositional attention $\mathcal{H}(\varphi)$ must contain $\ell(S')$ as an event. By $(v,\ell (S')\wedge \bigwedge_{a\in Ag}(A_a T'_a\wedge \neg A_a(S'\setminus T'_a))\in W''$ we know that $(\mathcal{M},v)\vDash \ell (S')\wedge \bigwedge_{a\in Ag}(A_a T'_a\wedge \neg A_a(S'\setminus T'_a))$, and in particular that $(\mathcal{M},v)\vDash \ell (S')$, by definition of product update. This implies that $(v,\ell (S'))\in W'$. 
	
	We want to show that $((w,\ell (S)),(v,\ell (S')))\in R_b'$, i.e., property (1b). By our initial assumption that  $((w,\ell (S)\wedge \bigwedge_{a\in Ag}(A_a T_a\wedge \neg A_a(S\setminus T_a))),(v,\ell (S')\wedge \bigwedge_{a\in Ag}(A_a T'_a\wedge \neg A_a(S'\setminus T'_a))))\in R''_b$ and definition of product update, it follows that $(w,v)\in R_b$. Hence, what remains to show is that there is a conditioned edge $(\ell (S){:}A_b S'\wedge A_b(S\setminus S'), \ell(S'){:}A_b S')\in Q_b$ such that $(\mathcal{M},w)\vDash A_b S'\wedge A_b(S\setminus S')$ and $(\mathcal{M},v)\vDash A_b S'$, as then we would know, by definition of edge-conditioned product update, that $((w,\ell (S)),(v,\ell (S')))\in R_b'$.
	
	By definition of edge-conditioned event model for propositional attention, we know that for all subsets $T$ of $S$, there is a conditioned edge $(\ell (S){:}A_b T\wedge A_b(S\setminus T), \ell(T){:}A_b T)\in Q_b$. As $\ell (S')\wedge \bigwedge_{a\in Ag}(A_a T'_a\wedge \neg A_a(S'\setminus T'_a)\in E'$, then by definition of standard event model for propositional attention, $S'\subseteq S$. Hence, $(\ell (S){:}A_b S'\wedge A_b(S\setminus S'), \ell(S'){:}A_b S')\in Q_b$. 	Now we only need to show that $(\mathcal{M},w)\vDash A_b S'\wedge A_b(S\setminus S')$ and $(\mathcal{M},v)\vDash A_b S'$.

Since $(w,\ell (S)\wedge \bigwedge_{a\in Ag}(A_a T_a\wedge \neg A_a(S\setminus T_a)))\in W''$ then $\ell (S)\wedge \bigwedge_{a\in Ag}(A_a T_a\wedge \neg A_a(S\setminus T_a))\in E'$. 
	By $((w,\ell (S)\wedge \bigwedge_{a\in Ag}(A_a T_a\wedge \neg A_a(S\setminus T_a))),(v,\ell (S')\wedge \bigwedge_{a\in Ag}(A_a T'_a\wedge \neg A_a(S'\setminus T'_a))))\in R''_b$ and definition of standard event model for propositional attention, it must be the case that $(A_b S'\wedge \neg A_b(S\setminus S')$ is a subformula of $\ell (S)\wedge \bigwedge_{a\in Ag}(A_a T_a\wedge \neg A_a(S\setminus T_a))$.
	Then, by reasoning as in the proof of [Forth], we know that for all $g\in E'$ such that $(\ell (S)\wedge \bigwedge_{a\in Ag}(A_a T_a\wedge \neg A_a(S\setminus T_a)),g)\in Q'_b$, $\ell(S')\wedge A_b S' \in g$ by \textsc{Attentiveness}, and $\ell(S\setminus S')\not\in g, \ell(At(\varphi)\setminus S)\not\in g$ by \textsc{Inertia}. By assumption we have that $((w,\ell (S)\wedge \bigwedge_{a\in Ag}(A_a T_a\wedge \neg A_a(S\setminus T_a))),(v,\ell (S')\wedge \bigwedge_{a\in Ag}(A_a T'_a\wedge \neg A_a(S'\setminus T'_a))))\in R''_b$, which by definition of standard product update implies that $(\ell (S)\wedge \bigwedge_{a\in Ag}(A_a T_a\wedge \neg A_a(S\setminus T_a)),\ell (S')\wedge \bigwedge_{a\in Ag}(A_a T'_a\wedge \neg A_a(S'\setminus T'_a)))\in Q'_b$. So we must have $A_bS'\in \ell (S')\wedge \bigwedge_{a\in Ag}(A_a T'_a\wedge \neg A_a(S'\setminus T'_a))$. 
	Then, as we know that $(\mathcal{M},v)\vDash \ell (S')\wedge \bigwedge_{a\in Ag}(A_a T'_a\wedge \neg A_a(S'\setminus T'_a))$, we also know that $(\mathcal{M},v)\vDash A_b S'$. 
	Moreover, by $(w,\ell (S)\wedge \bigwedge_{a\in Ag}(A_a T_a\wedge \neg A_a(S\setminus T_a)))\in W''$ and definition of product update, $(\mathcal{M},w)\vDash\ell (S)\wedge \bigwedge_{a\in Ag}(A_a T_a\wedge \neg A_a(S\setminus T_a))$, and in particular $(\mathcal{M},w)\vDash A_b S'\wedge \neg A_b(S\setminus S')$, as desired.

Lastly, we show that $((w,\varphi),(w,e))\in Z$. Recall that $\varphi\in e$, and that for all $a\in Ag$, there is a unique $T_a\subseteq At(\varphi)$ such that $(\mathcal{M},w)\models \bigwedge_{a\in Ag}(\bigwedge_{p_i\in T_a}A_a p_i\wedge \bigwedge_{p_i\in At(\varphi)\setminus T_a} \neg A_a p_i)$, which implies that $\bigwedge_{a\in Ag}(\bigwedge_{p_i\in T_a}A_a p_i\wedge \bigwedge_{p_i\in At(\varphi)\setminus T_a} \neg A_a p_i)\in e$ by definition of product update. Hence, it must be the case that $e=\varphi\wedge \bigwedge_{p_i\in S}\ell(p_i)\bigwedge_{a\in Ag}(\bigwedge_{p_i\in T_a}A_a p_i\wedge \bigwedge_{p_i\in At(\varphi)\setminus T_a} \neg A_a p_i)$. 

As it holds that for all the literals in $\varphi$, they are constructed from atoms belonging to $At(\varphi)$ and for all $a\in Ag, T_a\subseteq At(\varphi)$, then clearly $((w,\varphi),(w,e))\in Z$, by definition of $Z$.
\end{proof}

\begin{proof}[\textbf{Proof of Theorem
	\ref{prop: equivalence SEM into EDGY}}] (\emph{$\mathcal{E} \in \sem(\mathcal{L}_\text{DEL})$ 
	is update equivalent to
	$T_1(\mathcal{E})$.})
	  Let $\mathcal{E} = (E,Q,pre)$ be a standard event model for $\mathcal{L}_\text{DEL}$. 	
	 We show that for any Kripke model $\mathcal{M}$ with atom set $At(\mathcal{L}_\text{DEL}) = P$, $\mathcal{M}\otimes\mathcal{E}$ is isomorphic to $\mathcal{M}\otimes T_1(\mathcal{E})$. As isomorphism is a special case of bisimulation, it then follows that $\mathcal{M}\otimes\mathcal{E}$ is bisimilar to $\mathcal{M}\otimes T_1(\mathcal{E})$, i.e. $\mathcal{E}$ is update equivalent to $T_1(\mathcal{E})$. 
		
		The proof is induction on the level of $\mathcal{E}$ in the language hierarchy (cf.\ Section \ref{section:language hierarchies} above). The base case is 
		$cond(\mathcal{E})\subseteq \mathcal{L}_\text{DEL}^0$, i.e. 
		all preconditions in $\mathcal{E}$ are static formulas. Inspecting Definition~\ref{def: translation SEM into ECEM}, we immediately see that for static formulas $\varphi$, we have $T_2(\varphi) = \varphi$. By definition of $T_1$, we have $T_1(\mathcal{E}) = (E,Q',pre')$ where for each $a \in Ag$, $Q'_a = \{(e{:}\top, f{:}\top): (e,f) \in Q_a \}$, and for each $e \in E$, $pre'(e) = T_2(pre(e))$. Since $pre(e)$ is a static formula, we have $T_2(pre(e)) = pre(e)$, and thus $pre' = pre$. In other words, $T_1(\mathcal{E}) = (E,Q',pre)$.  
		
		Let $\mathcal{M}=(W,R,V)$ be any Kripke model with atom set $P$. 
		Notationally, let $\mathcal{M}\otimes \mathcal{E}=(W',R',V')$ and $\mathcal{M}\otimes T_1(\mathcal{E})=(W'',R'',V'')$. We need to show that $\mathcal{M}\otimes\mathcal{E}$ is isomorphic to $\mathcal{M}\otimes T_1(\mathcal{E})$, i.e. we need to show that there is a bijection $f:W' \rightarrow W''$ that satisfies the two properties of isomorphism (Def.~\ref{def:isomorphism} above).  
%
		To that goal, first notice that $W'=W''$. This is because $(u,g)\in W'$ iff (by def.\ of standard product update) $(u,g)\in W\times E$ and  $(\mathcal{M},u)\vDash pre(g)$ iff (by def. of edge-conditioned product update) $(u,g)\in W''$.
		
		Now define $f$ as the identity function on $W'$, i.e. $f(w,e)=(w,e)$ for all $(w,e)\in W'$. Clearly, $f$ is a bijection, and since the valuation is defined in the same way in the product update for standard event models and  the product update for edge-conditioned event models, then we also have that $p\in V'(w,e)$ 
		iff $p\in V''(f(w,e))$
		for all $p \in P$. Hence, it only remains to show that property 2 of Def.~\ref{def:isomorphism} 
		holds, i.e.\ that $f$ preserves the accessibility relations. 
		
		Notice that the following sequence of iff's is the case, for all $a\in Ag$: $((w,e),(v,f))\in R_a'$ iff (by definition of standard product update) $(w,v)\in R_a$ and $(e,f)\in Q_a$ iff (by definition of transformation $T_1$) $(w,v)\in R_a$ and $(e{:}\top,f{:}\top)\in Q'_a$ iff (by definition of edge-conditioned product update) $((w,e),(v,f))\in R_a''$. This completes the proof of the base case.

		For the induction step, suppose $\mathcal{E}$ is of level $i+1$ in the language hierarchy, i.e.\ $cond(\mathcal{E})\subseteq \mathcal{L}_\text{DEL}^{i+1}$. The induction hypothesis is the following assumption:
		For all Kripke models $\mathcal{M}$ and all  $\mathcal{E} \in \sem(\mathcal{L}_\text{DEL})$ 
		with $cond(\mathcal{E})\subseteq \mathcal{L}_\text{DEL}^{i}$, 
		$\mathcal{M} \otimes \mathcal{E}$ and $\mathcal{M} \otimes T_1(\mathcal{E})$ are isomorphic $(\dagger)$. 
		By definition of $T_1$, we have as before that for each $e \in E$, $pre'(e) = T_2(pre(e))$. 
		We again let $\mathcal{M} = (W,R,V)$ be a Kripke model, and use the same notation as above for $\mathcal{M} \otimes \mathcal{E}$ and $\mathcal{M} \otimes T_1(\mathcal{E})$. 
		We first show the following claim.

		\medskip \noindent
		\emph{Claim 1}. Given $w \in W$ and $\gamma \in \mathcal{L}^{i+1}_\text{DEL}$, then $(\mathcal{M},w) \vDash \gamma$ iff $(\mathcal{M},w) \vDash T_2(\gamma)$. 
		
		\medskip
		\noindent \emph{Proof of Claim 1.} 
		The proof is by induction on the structure of $\gamma$: 
		\begin{itemize}
		   \item $\gamma = p$. We have $T_2(p) = p$, and the conclusion follows.
		   \item $\gamma = \neg \varphi$. Since $T_2(\neg \varphi) = \neg T_2(\varphi)$, the induction hypothesis immediately gives the result.
           \item $\gamma = \varphi \land \psi$. Since $T_2(\varphi \land \psi) = T_2(\varphi) \land T_2(\psi)$, the induction hypothesis immediately gives the result.
           \item $\gamma = B_a \varphi$. We have $T_2(B_a \varphi) = B_a T_2(\varphi)$, and again the induction hypothesis gives the result: $(\mathcal{M},w) \vDash B_a \varphi$ iff for all $v$ with $(w,v) \in R_a$, $(\mathcal{M},v) \vDash \varphi$ iff (by induction hypothesis) for all $v$ with $(w,v) \in R_a$, $(\mathcal{M},v) \vDash T_2(\varphi)$ iff $(\mathcal{M},w) \vDash B_a T_2(\varphi)$.
           \item $\gamma = [(\mathcal{E}',e')]\varphi$, case $(\mathcal{M},w) \nvDash pre(e')$. In this case $(\mathcal{M},w) \vDash \gamma$ by definition of the semantics of the dynamic modality (since $(\mathcal{E}',e')$ is not applicable in $(\mathcal{M},w)$). By definition of $T_2$, we have $T_2([(\mathcal{E}',e')]\varphi) = [(T_1(\mathcal{E}'),e')]T_2(\varphi)$. By definition of applicability for edge-conditioned event models, $(T_1(\mathcal{E}'),e')$ is also not applicable in $(\mathcal{M},w)$, and hence $(\mathcal{M},w) \vDash [(T_1(\mathcal{E}'),e')]T_2(\varphi)$, by the semantics of the dynamic modality.

		   \item $\gamma = [(\mathcal{E}',e')]\varphi$, case $(\mathcal{M},w) \vDash pre(e')$.  
		 As above, by definition of $T_2$, we have $T_2([(\mathcal{E}',e')]\varphi) = [(T_1(\mathcal{E}'),e')]T_2(\varphi)$. 
		   Since $\gamma \in \mathcal{L}^{i+1}_\text{DEL}$, we have that $cond(\mathcal{E}') \subseteq \mathcal{L}^i_\text{DEL}$. By $(\dagger)$, we then get that $\mathcal{M} \otimes \mathcal{E}'$ and $\mathcal{M} \otimes T_1(\mathcal{E}')$ are isomorphic. 
		   Thus: $(\mathcal{M},w) \vDash \gamma$ iff $(\mathcal{M},w) \vDash 
 		   [(\mathcal{E}',e')]\varphi$ iff 
		   $(\mathcal{M} \otimes \mathcal{E}', (w,e')) \vDash \varphi$ iff (by induction hypothesis)  $(\mathcal{M} \otimes \mathcal{E}', (w,e')) \vDash T_2(\varphi)$ iff (since $\mathcal{M} \otimes \mathcal{E}'$ and $\mathcal{M} \otimes T_1(\mathcal{E}')$ are isomorphic) 
		   $(\mathcal{M} \otimes T_1(\mathcal{E}'), (w,e')) \vDash T_2(\varphi)$ iff $(\mathcal{M},w) \vDash [(T_1(\mathcal{E}'),e')]T_2(\varphi)$ iff $(\mathcal{M},w) \vDash T_2(\gamma)$.
		\end{itemize}      
		This concludes the proof of the claim.

		\medskip
		By Claim 1, it in particular follows that for all $w \in W$ and $e \in E$, $(\mathcal{M},w) \vDash pre(e)$ iff $(\mathcal{M},w) \vDash T_2(pre(e))$, i.e. $(\mathcal{M},w) \vDash pre(e)$ iff $(\mathcal{M},w) \vDash pre'(e)$. 
		Hence $W' = \{ (u,g) \in W \times E \mid (\mathcal{M},u) \vDash pre(g) \} = 
		\{ (u,g) \in W \times E \mid (\mathcal{M},u) \vDash pre'(g) \} = W''$. Thus, as in the base case, we can define $f$ as the identity function on $W'$, and we again get that $p \in V'(w,e)$ iff $p \in V''(f(w,e))$ for all $p \in P$. It then again only remains to show that $f$ preserves the accessibility relations, which is proved exactly as in the base case. 
	\end{proof}

\begin{proof}[\textbf{Proof of Theorem \ref{prop: equivalence EDGY to SEM}}](%
	\emph{$\mathcal{C} \in \ecem(\mathcal{L}_\text{ECM})$ is update equivalent to $T_1'(\mathcal{C})$.})
	The proof follows the structure of the proof of Theorem~\ref{prop: equivalence SEM into EDGY}. 
		We first recall the definition of the transformation $T'_1$ (Def.~\ref{def: transformation edgys to SEMs}): 
	 For any edge-conditioned event model $(E,Q,pre)$ for $\mathcal{L}_\text{ECM}$ and any $e \in E$, let $\Phi(e)$ be the set of source and target conditions at $e$, i.e.  $\Phi(e)=\{\varphi\in\mathcal{L}_\text{ECM}\colon (e{:}\varphi,f{:}\psi)\in Q_a\}\cup \{\psi\in\mathcal{L}_\text{ECM}\colon (f{:}\varphi,e{:}\psi)\in Q_a\}$. Set $\Phi'(e)=\Phi(e)\cup \{\neg\varphi\colon \varphi\in \Phi(e)\}$ and let $\mathsf{mc}(e)$ denote the set of maximally consistent subsets of $\Phi'(e)$, i.e. $\mathsf{mc}(e)$ is such that for all $\Gamma\in \mathsf{mc}(e)$ the following holds: 1) $\Gamma\subseteq \Phi'(e)$; 2) $ \not\vDash\Gamma\rightarrow \bot$; 3) for all $\Gamma'\subseteq \Phi'(e)$ with $\Gamma\subset \Gamma'$, $\vDash\Gamma'\rightarrow  \bot$.
	By mutual recursion we define mappings $T_1': \ecem(\mathcal{L}_\text{ECM}) \to \sem(\mathcal{L}_\text{DEL})$ and 
	$T_2':\mathcal{L}_\text{ECM} \to \mathcal{L}_\text{DEL}$. The mapping $T_2'$ is given by:
	\[
		\begin{array}{ll} 
		 T_2'(p) = p, \text{ for } p \in P 
		 &
		 T_2'(\neg \varphi) = \neg T_2' (\varphi) \\
		 T_2'(\varphi \land \psi) = T_2'(\varphi) \land T_2'(\psi) &
		 T_2'(B_a \varphi) = B_a T_2'(\varphi) \\
		 T_2'([(\mathcal{E},e)]\varphi) = [(T_1'(\mathcal{E}),e) ] T_2'(\varphi) 
		\end{array}
	   \]
    The mapping $T_1'$ is defined by 
	$T_1'(E,Q,pre) = (E',Q',pre')$ where:
	\[
	\begin{array}{l}
		E'= \{(e,\Gamma)\colon e\in E, \Gamma\in \mathsf{mc}(e)\}, \\
		Q'_a  =  \{((e,\Gamma),(e',\Gamma'))\in E'\times E' \colon (e{:}\varphi, e'{:}\varphi')\in Q_a, \\
		\qquad \quad \varphi\in \Gamma, \varphi'\in \Gamma'\}, \\
		pre'(e,\Gamma) = T_2'(pre(e)\wedge \bigwedge\Gamma).
	\end{array}
	\]

    Let $\mathcal{C}=(E,Q,pre) \in \ecem(\mathcal{L}_\text{ECM})$ 
	and $T'_1(\mathcal{C})=(E',Q',pre')$. Note that $At(\mathcal{L}_\text{ECM}) = P$.
	We need to show that for any Kripke model $\mathcal{M}=(W,R,V)$ with atom set $P$, 
	 $\mathcal{M}\otimes\mathcal{C}=(W',R',V')$ is isomorphic to $\mathcal{M}\otimes{T'_1(\mathcal{C})}=(W'',R'',V'')$. This is by induction on the level of $\mathcal{C}$ in the hieararchy of languages, with the base case being $cond(\mathcal{C}) \subseteq \mathcal{L}_\text{ECM}^0$.

The proof of the base case proceeds similarly to the proof of Theorem 4.7 in Kooi and Rennes~[\citeyear{kooi2011arrow}]. First we prove the following claim.

\medskip \noindent
\emph{Claim 1} For every $(w,e)\in W'$, there is a unique $\Gamma_{(w,e)}\in \mathsf{mc}(e)$ such that $(\mathcal{M},w)\vDash \bigwedge\Gamma_{(w,e)}$.

\medskip \noindent 
\emph{Proof of Claim 1}. 
Consider an arbitrary $(w,e)\in W'$ and let $\langle \varphi_i \rangle^n_{i=0}$ be an enumeration of all formulas in $\Phi(e)$. For all $i\in \mathbb{N}$ with $i \le n$, define the formula $\psi^{(w,e)}_i$ by 

$\psi^{(w,e)}_i=\begin{cases}\varphi_i \quad\quad\hspace{4pt}\text{ if } (\mathcal{M},w)\vDash \varphi_i \\
	\neg\varphi_i\quad\quad\text{if } (\mathcal{M},w)\not\vDash \varphi_i.
\end{cases}$

Define $\Gamma_{(w,e)}=\{\psi_i^{(w,e)}\colon i\in \mathbb{N}, i\le n\}$. It follows by construction that $\Gamma_{(w,e)}$ is a maximally consistent set of source and target conditions of $e$, i.e. $\Gamma_{(w,e)}\in\mathsf{mc}(e)$. Moreover, $(\mathcal{M},w)\vDash \bigwedge\Gamma_{(w,e)}$ by construction. Now consider some $\Gamma\in\mathsf{mc}(e)$ such that $\Gamma\not=\Gamma_{(w,e)}$. Then it follows by maximal consistency of $\Gamma$ that there is a $j\in \mathbb{N}$ with $j\le n$ such that either $\varphi_j\in \Gamma$ and $\neg\varphi_j\in \Gamma_{(w,e)}$, or  $\neg\varphi_j\in \Gamma$ and $\varphi_j\in \Gamma_{(w,e)}$, but in each case it holds that $(\mathcal{M},w)\not\vDash \bigwedge\Gamma$. Hence, for any $(w,e)\in W'$, $\Gamma_{(w,e)}$ is the unique set in $\mathsf{mc}(e)$ that satisfies $(\mathcal{M},w)\vDash \bigwedge\Gamma_{(w,e)}$. This concludes the proof of the claim. 

\medskip
To show isomorphism, we need to show that there exists a bijection between $W'$ and $W''$ that satisfies properties 1 and 2 of Definition~\ref{def:isomorphism} above. Let $f:W'\rightarrow W''$ be defined by $f(w,e)=(w,(e,\Gamma_{(w,e)}))$, i.e.\ where $\Gamma_{(w,e)}$ is the unique set in $\mathsf{mc}(e)$ such that $(\mathcal{M},w)\vDash \bigwedge\Gamma_{(w,e)}$. To show that $f$ is an isomorphism, we first need to show that $f$ is a bijection, namely that it is injective and surjective.

Injectivity: We want to show that if $(w,e)\not=(w',e')$ then $f(w,e)\not=f(w',e')$. We prove the contrapositive. So suppose $f(w,e)=f(w',e')$. Then we get that $(w,(e,\Gamma_{(w,e)}))= f(w,e) = f(w',e') = (w',(e',\Gamma_{(w',e')}))$, immediately implying $w = w'$ and $e = e'$.
Surjectivity: We want to show that for all $(w,(e,\Gamma))\in W''$ there is a $(w,e)\in W'$ such that $f(w,e)=(w,(e,\Gamma))$. 
Let $(w,(e,\Gamma))\in W''$. 
Then $(\mathcal{M},w)\vDash pre'(e,\Gamma)$ by definition of product update, and hence $(\mathcal{M},w)\vDash T_2' (pre(e)\wedge \bigwedge\Gamma)$ by definition of $T_1'(\mathcal{C})$. Since $cond(\mathcal{C}) \subseteq \mathcal{L}_\text{ECM}^0$, $pre(e)\wedge \bigwedge\Gamma$ is a static formula, and we then have $T_2'(pre(e)\wedge \bigwedge\Gamma)=pre(e)\wedge \bigwedge\Gamma$. Hence, $(\mathcal{M},w)\vDash pre(e)\wedge \bigwedge\Gamma$. As $e\in E$ and $(\mathcal{M},w)\vDash pre(e)$, then $(w,e)\in W'$. It follows that $f(w,e) = (w,(e,\Gamma_{(w,e)}))$, so now we only need to prove that $\Gamma=\Gamma_{(w,e)}$. That follows immediately from Claim 1, since $(\mathcal{M},w)\vDash \bigwedge\Gamma$.

Now we know that $f$ is a bijection. Left is to prove that $f$ satisfies properties 1 and 2 of Definition~\ref{def:isomorphism} above.
For the first property, consider an arbitrary $(w,e)\in W'$ and  $p\in P$. Then $p\in V'(w,e)$ iff (by definition of edge-conditioned product update) $p\in V(w)$ iff (by definition of standard product update) $p\in V''(w,(e,\Gamma_{(w,e)}))$. Hence, $p\in V'(w,e)$ iff  $p\in V''(f(w,e))$. 

For the second property, we want to show that $((w,e),(v,g))\in R'_a$ iff $(f(w,e),f(v,g))\in R''_a$, for all $(w,e),(v,g)\in W'$ and $a\in Ag$. So let $((w,e),(v,g))\in R'_a$ for arbitrary $(w,e),(v,g)\in W'$ and $a\in Ag$. By definition of edge-conditioned product update, $(w,v)\in R_a$ and there exists $\varphi,\psi$ such that $(e{:}\varphi, g{:}\psi)\in Q_a$ and $(\mathcal{M},w)\vDash \varphi$ and $(\mathcal{M},v)\vDash \psi$. By definition of $T_1'$, since $(e{:}\varphi, g{:}\psi)\in Q_a$, then $((e,\Gamma),(g,\Gamma'))\in Q'_a$, for all events $(e,\Gamma),(g,\Gamma') \in E'$ that are such that $\varphi\in\Gamma$ and $\psi\in\Gamma'$. As $\varphi$ is a source condition of $e$ and $(\mathcal{M},w) \vDash \varphi$, we get  
$\varphi\in \Gamma_{(w,e)}$. Similarly, as 
$\psi$ is a target condition of $g$ and $(\mathcal{M},v) \vDash \psi$, we get 
$\psi\in\Gamma_{(v,g)}$. Hence  $((e,\Gamma_{(w,e)}),(g,\Gamma_{(v,g)}))\in Q'_a$. 
By definition of standard product update and by $(w,v)\in R_a$, this implies that $((w,(e,\Gamma_{(w,e)})),(v,(g,\Gamma_{(v,g)})))\in R''_a$, i.e.\ 
 $(f(w,e),f(v,g))\in R''_a$, as required. 

%

For the other direction, let $(f(w,e),f(v,g))\in R''_a$, i.e. $((w,(e,\Gamma_{(w,e)})),(v,(g,\Gamma_{(v,g)})))\in R''_a$. By definition of product update, $(w,v)\in R_a$ and $((e,\Gamma_{(w,e)}),(g,\Gamma_{(v,g)}))\in Q'_a$. By definition of $T_1'$, $(e{:}\varphi, g{:}\psi)\in Q_a$ with $\varphi\in \Gamma_{(w,e)}$ and $\psi\in \Gamma_{(v,g)}$. By Claim 1, $(\mathcal{M},w)\vDash \bigwedge \Gamma_{(w,e)}$ and$(\mathcal{M},v)\vDash \bigwedge \Gamma_{(v,g)}$, and so we know that $(\mathcal{M},w)\vDash \varphi$ and $(\mathcal{M},v)\vDash \psi$. Since $(w,v)\in R_a$ and since there are $\varphi$ and $\psi$ such that $(e{:}\varphi,f{:}\psi)\in Q_a$ and $(\mathcal{M},w)\vDash \varphi$ and  $(\mathcal{M},v)\vDash \psi$, then by definition of product update for edge-conditioned event models, $((w,e),(v,g))\in R'_a$, as required.

Now we have covered the base case. The induction step follows the same strategy as in the proof of Theorem~\ref{prop: equivalence SEM into EDGY}. We consider $\mathcal{C}$ of level $i+1$, and the induction hypothesis is then: For all Kripke models $\mathcal{M}$ and all edge-conditioned event models $\mathcal{C}$ of level $i$, $\mathcal{M}\otimes\mathcal{C}$ and 
$\mathcal{M}\otimes T'_1(\mathcal{C})$ are isomorphic ($\ddagger$). 

\medskip
\noindent \emph{Claim 2} For all $w \in W$ and $\gamma \in \mathcal{L}_\emph{ECM}^{i+1}$, $(\mathcal{M},w) \vDash \gamma$ iff $(\mathcal{M},w) \vDash T_2'(\gamma)$. 

\medskip
\noindent \emph{Proof of Claim 2}. 
Note that: 1) ($\ddagger$) is the same condition as ($\dagger$) in the proof of Theorem~\ref{prop: equivalence SEM into EDGY}, except we have replaced $T_1$ by $T_1'$ and $\mathcal{E}$ by $\mathcal{C}$; 2) $T_2'$ is the same translation as $T_2$, except we have replaced $T_i$ by $T_i'$ and $\mathcal{E}$ by $\mathcal{C}$. In other words, the proof of Claim 1 in Theorem~\ref{prop: equivalence SEM into EDGY} immediately translates into this setting, when replacing $T_i$ by $T_i'$, $\mathcal{E}$ by $\mathcal{C}$, and $(\dagger)$ by ($\ddagger$). This concludes the proof of the claim. 

Now note that the only time in the proof of the base case that we used the base case assumption $cond(\mathcal{C}) \subseteq \mathcal{L}_\text{ECM}^0$ was to conclude from $(\mathcal{M},w) \vDash T_2'(pre(e) \land \bigwedge \Gamma)$ to 
$(\mathcal{M},w) \vDash pre(e) \land \bigwedge \Gamma$. Since $\mathcal{C}$ in the induction step has level $i+1$, then $cond(\mathcal{C}) \subseteq \mathcal{L}_\text{ECM}^{i+1}$. This implies $pre(e) \in \mathcal{L}_\text{ECM}^{i+1}$ as well as $\Gamma \subseteq \mathcal{L}_\text{ECM}^{i+1}$. Hence also $pre(e) \land \bigwedge \Gamma \in \mathcal{L}_\text{ECM}^{i+1}$, and Claim 2 then gives us that we can conclude from $(\mathcal{M},w) \vDash T_2'(pre(e) \land \bigwedge \Gamma)$ to 
$(\mathcal{M},w) \vDash pre(e) \land \bigwedge \Gamma$. In other words, the same proof as for the base case goes through with this extra line of argumentation.  
\end{proof}

%
\begin{proof}[\textbf{Proof of Theorem~\ref{thm:linear-growth-SEM}}](\emph{For any $\mathcal{E} \in \sem(\mathcal{L}_\text{DEL})$, $T_1(\mathcal{E})$ has size $O(|\mathcal{E}|)$.}%
	) Let $\mathcal{E}=(E,Q,pre) \in \sem(\mathcal{L}_\text{DEL})$  and $T_1(\mathcal{E})=(E,Q',pre')$. 
	The proof is by induction on the level of $\mathcal{E}$ in the language hierarchy. To make the induction step go through, we prove a slightly stronger claim. Define 
	$size(Q_a') = \Sigma_{a\in Ag}( |Q_a'| + \Sigma_{(e:\varphi, f:\psi)\in Q_a'} (|\varphi| + |\psi|))$. 
	Note that then $|\mathcal{E}| = |E| + \Sigma_{a \in Ag} |Q_a| + \Sigma_{e\in E} |pre(e)|$ and $|T_1(\mathcal{E})| = |E| + \Sigma_{a \in Ag} size(Q_a') + \Sigma_{e\in E} |pre'(e)|$, using the definitions of sizes in Section~\ref{section:sizes}. We are now going to prove that $size(Q_a') \leq 3|Q_a|$ and $|pre'(e)| \leq 3|pre(e)|$ for all $a \in Ag$ and $e \in E$, which then implies $|T_1(\mathcal{E})| \leq 3|\mathcal{E}|$, and hence $|T_1(\mathcal{E})| \in O(|\mathcal{E}|)$.

	
	We first consider the base case $cond(\mathcal{E}) \subseteq \mathcal{L}_\text{DEL}^0$. As in the proof of Theorem~\ref{prop: equivalence SEM into EDGY}, 
 we can then conclude $pre' = pre$, and hence $|pre'(e)| = |pre(e)|$ for all $e \in E$.	
	Furthermore, by definition of $T_1$, 
	$(e,f)\in Q_a$ iff $(e{:}\top,f{:}\top)\in Q'_a$. This implies $|Q_a| = |Q'_a|$ for all $a \in Ag$.
	Since all elements of $Q_a'$ are of the form $(e{:}\top,f{:}\top)$, all source conditions $\varphi$ and all target conditions $\psi$ of $Q_a'$ satisfy $\varphi = \psi = \top$, and hence $|\varphi| = |\psi| =1$. Thus 
	 $size(Q_a') = |Q_a'| + \Sigma_{(e{:}\varphi,f{:}\psi)\in Q_a'}(|\varphi|+|\psi|) = |Q_a| + 2|Q_a| \leq  3|Q_a|$, as required. 
    This concludes the base case. 

	For the induction step, suppose $cond(\mathcal{E}) \subseteq \mathcal{L}_\text{DEL}^{i+1}$, and assume the result holds when $cond(\mathcal{E}) \subseteq \mathcal{L}_\text{DEL}^{i}$. We have $T_1(\mathcal{E})=(E,Q',pre')$. As $E$ and $Q'$ are as in the base case, it suffices to prove that 
	$ |pre'(e)| \leq 3 |pre(e)|$ for all $e \in E$.
   Let $e \in E$ be given. 
	 Since $cond(\mathcal{E}) \subseteq \mathcal{L}_\text{DEL}^{i+1}$, we have $pre(e) \in \mathcal{L}_\text{DEL}^{i+1}$. 
	 By definition of $T_1$, $pre'(e) = T_2(pre(e))$. 
	 So we need to prove that $|T_2(pre(e))| \leq 3|pre(e)|$. We will prove the following more general claim.

	 \medskip
	 \noindent
	 \emph{Claim 1}. For any $\gamma \in \mathcal{L}_\text{DEL}^{i+1}$, $|T_2(\gamma)| \leq 3|\gamma|$.

	 \medskip
	 \noindent \emph{Proof of Claim 1}. The proof is by induction on the structure of $\gamma$. As in the proof of Theorem~\ref{prop: equivalence SEM into EDGY}, all cases except $\gamma = [(\mathcal{E'},e')] \varphi$ are trivial. So suppose that  $\gamma = [(\mathcal{E'},e')] \varphi$. Since $\gamma \in \mathcal{L}_\text{DEL}^{i+1}$, we have $cond(\mathcal{E'}) \subseteq \mathcal{L}_\text{DEL}^{i}$. By the induction hypothesis of our overall proof, we then get that $|T_1(\mathcal{E}')| \leq 3|\mathcal{E'}|$. We can now conclude that $|T_2(\gamma)| = |T_2([(\mathcal{E'},e')] \varphi)| = |[T_1(\mathcal{E'}),e']T_2(\varphi)| = |
	 [T_1(\mathcal{E'}),e']| + |T_2(\varphi)| \leq  
	 3|
	 [\mathcal{E'},e']| + 3|\varphi| = 3|[\mathcal{E'},e']\varphi| = 3|\gamma|$, where the inequality uses both the induction hypothesis of this claim as well as that $|T_1(\mathcal{E}')| \leq 3|\mathcal{E'}|$. This completes the proof of the claim.
	 
    From the claim we immediately get $|T_2(pre(e))| \leq 3|pre(e)|$, since $pre(e) \in \mathcal{L}_\text{DEL}^{i+1}$, and we are hence done.
\end{proof}

\begin{proof}[\textbf{Proof of Theorem~\ref{thm:exponential-succinctness}}]  (\emph{Exponential succinctness of edge-conditioned event models}.)
	Let 
	$\mathcal{H}(p)  = ((E,Q,id_E),p)$ be the edge-conditioned event model for propositional attention representing the revelation of $p$, over some set of agents $Ag$.  
	 The set of events is $E= \{p, \top\}$, so the event model only has a constant number of events, independent of the size of $Ag$. Consulting Definition~\ref{def:edge-conditioned-prop-attention}, we see that each agent $a \in Ag$ has one edge per choice of sets $T,S$ with $T \subseteq S \subseteq \{ p \}$, i.e., for $(T,S) = (\{p\},\{p\})$, $(T,S) = (\emptyset,\{p\})$, and $(T,S) = (\emptyset,\emptyset)$. The corresponding edges in $Q_a$ are $(p{:}A_a p, p{:} A_a p)$, $(p{:} \neg A_a p ,\top{:} \top)$ and $(\top{:} \top, \top{:} \top)$. This amounts to, for each agent $a \in Ag$, a loop at each of the events $p$ and $\top$ as well as an edge from $p$ to $\top$. The full model is illustrated in Fig.~\ref{figure:H(p)} of this document, using the same conventions as Fig.~\ref{fig2:comparison-right} (note that, as expected, the event model of Fig.~\ref{figure:H(p)} is a submodel of the event model of Fig.~\ref{fig2:comparison-right}: the submodel consisting of the events and edges not mentioning $q$). 
	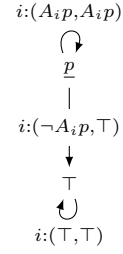
\begin{figure}     
	\begin{center}
	\begin{tikzpicture}
					\tikzset{scale=1, 
						world/.style={}, 
						actual/.style={}
					}
					\node [actual,fill=white] (!) at (0,0) {$\s \underline{p}$};
					\path (!) edge [-latex, looseness=6,in=70,out=110] (!) node [above, xshift=0pt, yshift=15pt] {$\s i:(A_i p, A_i p)$};

					\node [world] (T) at (0,-1.5) {$\s \top$};
										
					\path (T) edge [-latex, looseness=6,in=-110,out=-70] node [below,xshift=0pt] {$\s i:(\top, \top)$} (T);
					
					\path (!) edge[-latex,bend right=0] node[fill=white,above,yshift=-7pt,align=left,xshift=0pt]{$\s i:(\neg A_i p, \top)$}
					(T);
			
			\end{tikzpicture}
	\end{center}
   \caption{The event model $\mathcal{H}(p)$.} \label{figure:H(p)}
\end{figure} 
Each edge has constant size, also counting the source and target conditions, and as there are 3 such edges per agent, the total size of $\mathcal{H}(p)$ is $O(|Ag|)$. To make the argument a bit more formally precise, we can compute the size of $\mathcal{H}(p)$ using the relevant formula from Section~\ref{section:sizes}. First, since $|E| = 2$ (there are two events) and $|Q_a| = 3$ (each agent has 3 edges) and $|p| = |T| = 1$ (the event preconditions have length 1), we get:    
\[
\setlength{\arraycolsep}{0.1em}
\begin{array}{rl} 
|\mathcal{H}(p)| &= | E| + \!\!  \displaystyle\sum_{a \in Ag} \bigl(|Q_a|\ + \!\!\!\!\!\!\!\displaystyle\sum_{(e:\varphi, f:\psi)\in Q_a} \!\!\!\!(|\varphi| + |\psi|) \bigr) + \displaystyle\sum_{e \in E} |pre(e)| \\
&= 2 + \displaystyle\sum_{a \in Ag} \bigl(3   + \!\!\!\!\!\!\!\displaystyle\sum_{(e:\varphi, f:\psi)\in Q_a} \!\!\!\!(|\varphi| + |\psi|) \bigr) + 2 
\end{array}
\]
As all edge conditions have length $\leq 4$ (the longest is $\neg A_a p$), we further get:
\[
\setlength{\arraycolsep}{0.1em}
\begin{array}{rl} 
|\mathcal{H}(p)| &= 4 + \displaystyle\sum_{a \in Ag} \bigl(3   + \!\!\!\!\!\!\!\displaystyle\sum_{(e:\varphi, f:\psi)\in Q_a} \!\!\!\!(|\varphi| + |\psi|) \bigr)  \\
&\leq 4 + \displaystyle\sum_{a \in Ag} (3 +2 \cdot 4) =  4 + 11 |Ag| \in O(|Ag|).
\end{array}
\]
This proves 1. 

   	To prove 2, suppose $\mathcal{E}' = ((E',Q',pre'),E_d)$ is any pointed or multi-pointed standard event model update equivalent to $\mathcal{H}(p)$. We need to prove that $\mathcal{E}'$ has at least $2^{|Ag|}$ events. To do this, we will create $2^{|Ag|}$ distinct pointed Kripke models and show that $\mathcal{E}'$ must include a distinct event for each of these Kripke models. We will create one such Kripke model for each subset $Ag' \subseteq Ag$. To keep the notation simple, we for now fix an arbitrary subset $Ag' \subseteq Ag$, so that we do not need to index everything by $Ag'$ in the following. We then define the relevant pointed Kripke model $(\mathcal{M},w_1) = ((W,R,V),w_1)$ by letting $W = \{w_1,w_2\}$, $R_a = \{(w_1,w_2) \},$ for all $a \in Ag$, and $V(w_1) = \{p\} \cup \{ A_a p: a \in Ag' \}$, $V(w_2) = \emptyset$.
	We let $(\mathcal{M},w_1) \otimes \mathcal{H}(p) = ((W',R',V'),(w_1,p))$ (as $p \in V(w_1)$ and $pre(p) = p$, $\mathcal{H}(p)$ is applicable in $(\mathcal{M},w_1)$, and $(w_1,p)$ becomes the actual world of the updated model). 
    In the following, we call a world $w$ of a Kripke model a \emph{$\neg p$-world} if $p$ is false in $w$. 
	For a pointed or multi-pointed event model $\mathcal{E}$, we allow ourselves to also use $\mathcal{E}$ to denote the event model without the actual event(s). It will always be clear from the context whether $\mathcal{E}$ refers to the event model without or without the actual event(s) specified.  
	
	
 
   \medskip
   \noindent \emph{Claim 1}. 
   In $\mathcal{M} \otimes \mathcal{H}(p)$, the world $(w_1,p)$ has an $a$-edge to a $\neg p$-world iff $a \in Ag \setminus Ag'$. 
 
   \medskip
   \noindent \emph{Proof of Claim 1}.
   We first prove the left-to-right direction. So suppose $((w_1,p),(v,e)) \in R'_a$ and  $\mathcal{M} \otimes \mathcal{H}(p), (v,e) \vDash \neg p$ for some agent $a \in Ag$ and some world $(v,e) \in W'$. We need to show that $a \in Ag \setminus Ag'$. Since $(v,e)$ satisfies $\neg p$, we must have $e = \top$, as the only other event of $\mathcal{H}(p)$ has precondition $p$ (see Figure~\ref{figure:H(p)}). Thus $((w_1,p),(v,\top)) \in R'_a$.  
   From this, the definition of edge-conditioned product update (Def.~\ref{def:edge-conditioned-product-update}) gives us that there exists $\varphi,\psi$ such that $(p{:}\varphi,\top{:}\psi) \in Q_a$ and $(\mathcal{M},w_1) \models \varphi$ (recalling that $Q_a$ is the accessibility relation for $a$ of $\mathcal{H}(p)$). From the definition of $Q_a$ above, we can conclude that $\varphi = \neg A_a p$ and $\psi = \top$. Thus we get $(\mathcal{M},w_1) \models \neg A_a p$. Since $V(w_1) = \{p\} \cup \{ A_a p: a \in Ag' \}$, we must then have $a \not\in Ag'$. This proves the left-to-right direction. For the right-to-left direction, suppose $a \in Ag \setminus Ag'$. Then $(\mathcal{M},w_1) \models \neg A_a p$. We now have $(w_1,w_2) \in R_a$, $(p{:}\neg A_a p, \top{:}\top) \in Q_a$, $(\mathcal{M},w_1) \models \neg A_a p$ and $(\mathcal{M},w_2) \models \top$, which by the definition of edge-conditioned product update implies $((w_1,p),(w_2,\top)) \in R_a'$. As $V(w_2) = \emptyset$, $w_2$ is a $\neg p$-world, and then so is $(w_2,\top)$. 
   Thus $(w_1,p)$ has an $a$-edge to a $\neg p$-world, as required. This completes the right-to-left direction and hence the full proof of the claim.

   \smallskip
    Since $\mathcal{E}'$ is update equivalent to $\mathcal{H}(p)$ and $\mathcal{H}(p)$ is applicable in $(\mathcal{M},w_1)$, also $\mathcal{E}'$ must be applicable in $(\mathcal{M},w_1)$ (see Def.~\ref{def:update equivalence}). Thus $(\mathcal{M},w_1) \otimes \mathcal{E}' = ((W'',R'',V''),(w_1,e_d))$ for some $e_d \in E_d$. 
	By Def.~\ref{def:update equivalence}, we further get that  
	$(\mathcal{M},w_1) \otimes \mathcal{H}(p)$ is bisimilar to $(\mathcal{M},w_1) \otimes \mathcal{E}'$, and hence there exists a bisimulation relation $Z \subseteq W' \times W''$ satisfying the conditions of Def.~\ref{def: bisimulation}.
By Def.~\ref{def: bisimulation}, $((w_1,p),(w_1,e_d)) \in Z$
	(the actual worlds of the respective models are related by $Z$).   
		
	   \medskip
   \noindent \emph{Claim 2}. 
    The world $(w_1,p)$ has an $a$-edge to a $\neg p$-world iff $(w_1,e_d)$ has. 

	   \medskip
	      \noindent \emph{Proof of Claim 2}.
	We first prove the left-to-right direction. So suppose $((w_1,p),(w',e')) \in R_a'$ with $(w',e')$ satisfying $\neg p$. From [Forth] of the definition of bisimulation (Def.~\ref{def: bisimulation}), there exists $(w'',e'')$ such that $((w_1,e_d),(w'',e'')) \in R_a''$ and $((w',e'),(w'',e'')) \in Z$. From [Atom] applied to $((w',e'),(w'',e'')) \in Z$, we then get that also $(w'',e'')$ satisfies $\neg p$. This completes the proof of the left-to-right direction. The other direction is symmetric, using [Back] instead of [Forth]. This completes the proof of the claim.            

	\medskip 
	\noindent \emph{Claim 3}. The world $(w_1,e_d)$ has an $a$-edge to a $\neg p$-world iff $e_d$ has an $a$-edge to an event $f$ with $(\mathcal{M},w_2) \models pre(f)$.

	\medskip
	\noindent \emph{Proof of Claim 3}. Left-to-right: Suppose $(w_1,e_d)$ has an $a$-edge to a $\neg p$-world $(w',e')$. Then by definition of the standard product update, there is an $a$-edge from $w_1$ to $w'$ and an $a$-edge from $e_d$ to $e'$, and $(\mathcal{M},w') \models pre(e')$. Since $(w',e')$ is a $\neg p$-world, we must have $w' = w_2$. Letting $f = e'$, we thus have an $a$-edge from $e_d$ to $f$ with $(\mathcal{M},w_2) \models pre(f)$. This proves the left-to-right direction. For right-to-left, suppose $e_d$ has an $a$-edge to an event $f$ with $(\mathcal{M},w_2) \models pre(f)$. Note that by definition of $\mathcal{M}$, there is also an $a$-edge from $w_1$ to $w_2$. By definition of standard product update, we then get that $\mathcal{M} \otimes \mathcal{E}'$ has an $a$-edge from $(w_1,e_d)$ to $(w_2,f)$, and as $w_2$ and hence $(w_2,f)$ are $\neg p$-worlds, we are done. This completes the proof of the claim.

  \medskip
     Putting Claims 1, 2 and 3 together, we now have that $e_d$ has an $a$-edge to an event $f$ with $(\mathcal{M},w_2) \models pre(f)$ iff $a \in Ag \setminus Ag'$. Until now we kept the subset $Ag'$ fixed for simplicity, but we will now start to vary it. As the chosen $e_d$ might depend on the choice of $Ag'$, we will now refer to it as $e_d^{Ag'}$. What we have then proved is that for any subset $Ag' \subseteq Ag$, the event $e_d^{Ag'}$ has an $a$-edge to an
	 event $f$ with $(\mathcal{M},w_2) \models pre(f)$ iff $a \in Ag \setminus Ag'$. 	 
	 This shows that for any distinct subsets $Ag', Ag'' \subseteq Ag$, the events $e_d^{Ag'}$ and $e_d^{Ag''}$ must be distinct events of $\mathcal{E}'$, since they have distinct sets of outgoing edges (here it is important to note that $w_2$ satisfies no atoms and has no outgoing edges, so whether $(\mathcal{M},w_2) \models pre(f)$ holds or not is completely independent of the choice of $Ag'$). Thus there must be at least as many events in $\mathcal{E}'$ as there are distinct subsets of $Ag$, i.e., at least $2^{|Ag|}$ events. This proves 2. 
\end{proof}
Note that the proof above even proves that the number of \emph{actual} events of any standard event model $\mathcal{E}$ that is update equivalent to $\mathcal{H}(p)$ is exponential in the number of agents. However, this is also exactly what we see in the specific standard event model $\mathcal{F}(p)$ that is known to be update equivalent to $\mathcal{H}(p)$, see Fig.~\ref{fig2:comparison-left}.

	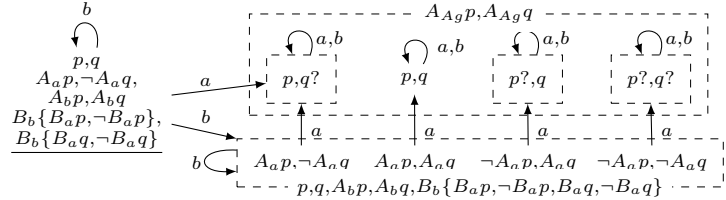
\begin{figure*}
	\begin{center}
		\begin{tikzpicture}[auto]\tikzset{ 
				deepsquare/.style ={rectangle,draw=black, inner sep=2.5pt, very thin, dashed, minimum height=3pt, minimum width=1pt, text centered}, 
				world/.style={}, 
				actual/.style={},
				cloudy/.style={cloud, cloud puffs=15, cloud ignores aspect, inner sep=0.5pt, draw}
			}
			
			\node [actual] (!) {$\underline{\s \substack{p, q \\A_a p, \neg A_aq, \\ A_b p, A_bq \\ B_b \{B_a p, \neg B_a p\},\\B_b\{B_a q, \neg B_a q\}}}$};
			\path (!) edge [-latex, loop, looseness=4,in=100,out=80] node [above] {$\s b$} (!);
			
			\node [world, right=of !, xshift=10pt, yshift=10pt] (1) {$\s p,q?$};
			\node [deepsquare, fit={(1)}](square1) {};
			\path (square1) edge [-latex, loop, looseness=4.6,in=110,out=70] (square1) node [right, xshift=2pt, yshift=15pt] {$\s a,b$};
			\path (!) edge [-latex] (square1) node [above, xshift=45pt, yshift=4pt] {$\s a$};
			
			\node [world, right=of 1, xshift=-5pt] (2) {$\s p,q$};
			\path (2) edge [-latex, looseness=6,in=110,out=60] (2) node [right, xshift=5pt, yshift=12pt] {$\s a,b$};

			\node [world, right=of 2, xshift=-5pt] (3) {$\s p?,q$};
			\node [deepsquare, fit={(3)}](square3) {};
			\path (square3) edge [-latex, looseness=4.6,in=110,out=70] (square3) node [right, yshift=15pt, xshift=2pt] {$\s a,b$};
			
			\node [world, right=of 3, xshift=-5pt] (4) {$\s p?,q?$};
			\node [deepsquare, fit={(4)}](square4) {};
			\path (square4) edge [-latex, looseness=4.6,in=110,out=70] (square4) node [right, yshift=15pt, xshift=2pt] {$\s a,b$};
			
			\node[left=of 1, xshift=29pt] (anchor1) {};
			\node[above=of 1, yshift=-20pt] (anchor2) {};
			\node[below=of 1, yshift=31pt] (anchor3) {};
			\node[right=of 4, xshift=-29pt] (anchor4) {};
			\node [deepsquare,fit={(anchor1)(anchor2)(anchor3)(anchor4)(1)(2)(3)(4)}](squareb) {};
			\node[fill=white] (square name 1) at (squareb.north) [yshift=-0pt]{$\s A_{Ag} p, A_{Ag}q$};

			\node [world, below=of 1, yshift=10pt] (1b) {$\s A_ap,\neg A_aq$};
			\path (1b.north) edge [-latex] (square1.south) node [right, yshift=5pt] {$\s a$};
			
			\node [world, below=of 2, right=of 1b, xshift=-26.5pt] (2b) {$\s A_ap,A_aq$};
			\path (2b.north) edge [-latex] (2.south) node [right, yshift=5pt] {$\s a$};
			
			\node [world, below=of 3, right=of 2b, xshift=-26.5pt] (3b) {$\s \neg A_ap,A_aq$};
			\path (3b.north) edge [-latex] (square3.south) node [right, yshift=5pt] {$\s a$};			
			
			\node [world, below=of 3, right=of 3b, xshift=-28pt] (4b) {$\s \neg A_ap,\neg A_aq$};
			\path (4b.north) edge [-latex] (square4.south) node [right, yshift=5pt] {$\s a$};			
			
			\node [deepsquare,fit={(1b)(2b)(3b)(4b)}](squareb) {};
			\node[fill=white] (square name 1) at (squareb.south) [yshift=-0pt]{$\s p, q, A_b p, A_b q, B_b \{B_a p, \neg B_a p,B_a q, \neg B_a q\}$};
			\path (squareb) edge [-latex, looseness=5,in=182,out=177] (squareb) node [left, xshift=-102pt, yshift=0pt] {$\s b$};
			
			\path (!) edge [-latex] (squareb.north west) node [above, xshift=45pt, yshift=-10pt] {$\s b$};
			
		\end{tikzpicture}
	\end{center}
	
	\caption{The attention model $(\mathcal{M}',w')$ for $\mathcal{L}_{\text{GA}}$. We use the same conventions as in the figures in the main text.
	}\label{figure: first general attention}
\end{figure*} 

\section{Comparisons of edge-conditioned event models with generalized arrow updates}

%


We now compare generalized arrow updates to edge-conditioned event models, following the same pattern as in the comparison of standard event models to edge-conditioned event models. The constructions below follow the same structure and ideas as above, so we will be more brief. 

\begin{definition}[Transformation of generalized arrow updates into edge-conditioned event models]
	By mutual recursion we define mappings $T''_1: \gau(\mathcal{L}_\text{GAU}) \to \ecem(\mathcal{L}_\text{ECM})$ and
	$T''_2: \mathcal{L}_\text{GAU} \to \mathcal{L}_\text{ECM}$. 
	Define $T''_1$ by $T''_1(O,\mathsf{a}) = (E,Q,pre)$ where:
	\[
	  \begin{array}{l} 
      E = O \\
      Q_a =\{(o{:}T_2''(\varphi),o'{:}T''_2(\varphi'))\colon (\varphi,o',\varphi’)\in \mathsf{a}_a(o)\},  a \in Ag \\
	  	  pre(o) = \top, \text{ for all } o \in E 
	  \end{array}
	\]
Define $T''_2$ by:
	\[ 
	  \begin{array}{l} 
	   T''_2(p) = p, \text{ for }  p \in P \\[0.1em]
	   T''_2(\neg \varphi) = \neg T''_2 (\varphi) \\[0.1em]
	   T''_2(\varphi \land \psi) = T''_2(\varphi) \land T_2(\psi) \\[0.1em]
	   T''_2(B_a \varphi) = B_a T''_2(\varphi) \\[0.1em] 
	   T''_2([(\mathcal{U},o)]\varphi) = [(T''_1(\mathcal{U}),o) ] T''_2(\varphi) 
	  \end{array}
	 \]
\end{definition}
In Def.~\ref{def:update equivalence}, update equivalence is defined where $\mathcal{E}$ and $\mathcal{E}'$ can either be standard or edge-conditioned event models. That definition trivially extends to allow the dynamic models to be generalized arrow updates.
\begin{theorem}
	$\mathcal{U} \in \gau(\mathcal{L}_\text{GAU})$
	is update equivalent to $T_1''(\mathcal{U})$.
\end{theorem}

\begin{proof}
    Let $\mathcal{U}=(O,\mathsf{a}) \in  \gau(\mathcal{L}_\text{GAU})$. Let $T_1''(\mathcal{U})=(E,Q,pre) \in \ecem(\mathcal{L}_\text{ECM})$ be its transformation into an edge-conditioned event model. Note that $At(\mathcal{L}_\text{GAU}) = At(\mathcal{L}_\text{ECM}) = P$.
	Let $\mathcal{M}=(W,R,V)$ be a Kripke model with atom set $P$.
	We want to show that there exists an isomorphism between $\mathcal{M}\otimes \mathcal{U}=(W',R',V')$ and $\mathcal{M}\otimes T_1''(\mathcal{U})=(W'',R'',V'')$. As always, we do induction on the level of $\mathcal{U}$ in the language hierarchy. 
	
	For the base case, we assume $cond(\mathcal{U}) \subseteq \mathcal{L}_\text{GAU}^0$.
	First notice that $W'= W''$ by the following identities: $W'=W\times O=W\times E=W''$, where the last identity holds because $pre(o)=\top$ for all $o\in E$, by definition of transformation $T_1''$.
	To see that there is an isomorphism, let $f: W'\rightarrow W''$ be defined by $f(w,o)=(w,o)$. We now want to show that properties 1 and 2 of Definition~\ref{def:isomorphism} above hold.  
	 Clearly, it holds that for all $(w,o)\in W'$ and $p\in P$, $p\in V'(w,o)$ iff $p\in V''(f(w,o))$. For the second property, let 
	$((w,o),(w',o'))\in R'_a$. We want to show that $((w,o),(w,o'))\in R_a''$. By definition of product update with generalized arrow updates, if 
	$((w,o),(w',o'))\in R'_a$ then $(\varphi , o', \varphi')\in \mathsf{a}_a(o), (\mathcal{M},w)\vDash \varphi,$ and $(\mathcal{M},w')\vDash \varphi'$. By definition of transformation $T_1''$, $(\varphi , o', \varphi')\in \mathsf{a}_a(o)$ implies that $(o{:}T_2''(\varphi),o'{:}T_2''(\varphi'))\in Q_a$. Since $cond(\mathcal{U}) \subseteq \mathcal{L}_\text{GAU}^0$, $\varphi$ and $\varphi'$ are static formulas and hence $T_2''(\varphi)=\varphi$ and $T_2''(\varphi')=\varphi'$. This means that $(o{:}\varphi,o'{:}\varphi')\in Q_a$, and hence $((w,o),(w',o'))\in R_a''$, since $(\mathcal{M},w) \vDash \phi$ and $(\mathcal{M},w') \vDash \phi'$. This proves that $((w,o),(w',o'))\in R'_a$ imples $((w,o),(w,o'))\in R_a''$. For the other direction, assume $((w,o),(w,o'))\in R_a''$. Then the exists $\varphi$ and $\varphi'$ such that $(o{:}\varphi,o'{:}\varphi')\in Q_a$ and $(\mathcal{M},w)\vDash \varphi$ and $(\mathcal{M},w')\vDash \varphi'$. By definition of $T_2''$, we can then conclude that the exists $\psi$ and $\psi'$ such that $\varphi=T_2''(\psi)$, $\varphi'=T_2''(\psi')$ and $(\psi,o',\psi') \in \mathsf{a}_a(o)$. Since $cond(\mathcal{U}) \subseteq \mathcal{L}_\text{GAU}^0$, we get $\varphi=T_2''(\psi) = \psi$ and $\varphi'=T_2''(\psi') = \psi'$. This means that $(\phi,o',\phi') \in \mathsf{a}_a(o)$. Since we now have $(\mathcal{M},w) \models \varphi$, $(\mathcal{M},w') \models \varphi'$ and $(\phi,o',\phi') \in \mathsf{a}_a(o)$, the definition of product update for generalized arrow updates give us that $((w,o),(w',o'))\in R'_a$. This completes the base case.
 	
    For the induction step, as before we have to deal with the case where $T_2''(\varphi) \neq \varphi$ for the relevant conditions, and appeal to the induction hypothesis. So let $\mathcal{U}$ be given with $cond(\mathcal{U}) \subseteq \mathcal{L}_\text{GAU}^{i+1}$ and assume that the result holds when $cond(\mathcal{U}) \subseteq \mathcal{L}_\text{GAU}^{i}$. As in the earlier proofs, the induction hypothesis gives us that the following claim holds.

	\medskip\noindent
	\emph{Claim 1}. For all $\gamma \in \mathcal{L}_\text{GAU}^{i+1}$, $(\mathcal{M},w) \models \gamma$ iff $(\mathcal{M},w) \models T_2(\gamma)$. 
	
	\medskip
	The only non-trivial case in the proof of this claim is when $\gamma = [(\mathcal{U}',o')] \varphi$, and that case is handled by the induction hypothesis, since then $cond(\mathcal{U}') \subseteq \mathcal{L}_\text{GAU}^{i}$. 
	
	Given Claim 1, we can now repeat the proof of the base case for the induction step, since we only apply $T_2''$ to formulas $\gamma$ in $cond(\mathcal{U}) \subseteq \mathcal{L}_\text{GAU}^{i+1}$, and for these we can always replace $T_2''(\gamma)$ with $\gamma$, since they are true in the same worlds. 
	This completes the proof.

\end{proof}

\begin{theorem}[Linear growth] 
	For any $\mathcal{U} \in \gau(\mathcal{L}_\text{GAU})$, $T''_1(\mathcal{U})$ has size $O(|\mathcal{U}|)$.
	\end{theorem}
\begin{proof}
  The proof is very similar to the proof of 
  Theorem~\ref{thm:linear-growth-SEM} above, so we only sketch the argument. Let $\mathcal{U}=(O,\mathsf{a})$ and $T_1''(\mathcal{U}) = \mathcal{C} = (E,Q,pre)$. 
  
  We need to inspect how much larger the translated model $\mathcal{C}$ can become. 
  Note that we have, from the size definition in Section~\ref{section:sizes}: 
  $$|\mathcal{C}| = | E| +   \sum_{a \in Ag} \bigl(|Q_a|\ + \!\!\!\!\sum_{(e:\varphi, f:\psi)\in Q_a} (|\varphi| + |\psi|) \bigr) + \sum_{e \in E} |pre(e)|$$	
	$$|\mathcal{U}| = | O| +   \sum_{a \in Ag, o \in O} |\mathsf{a}_a(o)|\ +\!\!\sum_{a \in Ag, o\in O,(\varphi, o,\psi)\in \mathsf{a}_a(o)} \!\!(|\varphi| + |\psi|)$$
 By definition of the translation $T_1''$, we have $E=O$ and hence $|E| = |O|$. We also have that $Q_a$ contains exactly one edge per choice of $(\varphi,o',\varphi') \in \mathsf{a}_a(o)$, and hence $|Q_a| = \Sigma_{o \in O} |\mathsf{a}_a(o)|$. In the base case of our induction proof (the case $cond(\mathcal{U}) \subseteq \mathcal{L}_\text{GAU}^0$), we have that $T_2''(\gamma) = \gamma$ for all the source and target conditions, and hence also $\Sigma_{a \in Ag,(e:\varphi, f:\psi)\in Q_a} (|\varphi| + |\psi|) = \Sigma_{a \in Ag, o \in O, (\varphi, o,\psi) \in \mathsf{a}_a(o)} (|\varphi| + |\psi|)$. Thus, the two models only differ in size by the addition of $\Sigma_{e\in E} |pre(e)|$ to the size of $\mathcal{C}$. Since $pre(e) = \top$ for all $e \in E$, we have $|pre(e)| = 1$ for all $e \in E$. Hence in the translation, we only add $|E| = |O|$ to the size of the input $\mathcal{U} = (O,\mathsf{a})$, and hence $|T_1''(\mathcal{U})| \leq  2|\mathcal{U}|$.

 For the induction step, we no longer have $T_2''(\gamma) = \gamma$ for the relevant conditions, but as in the proof of Theorem~\ref{thm:linear-growth-SEM} above, we can appeal to the induction hypothesis. The relevant source and target conditions only contain dynamic modalities $[(\mathcal{U'},o')]\varphi$ for which the induction hypothesis holds. The induction hypothesis gives that for such dynamic modalities that occur as subformulas in source and target conditions, we have $|T_1''(\mathcal{U}')| \leq 2 |\mathcal{U}'|$. This means that the size of the source and target conditions in the translation at most doubles in size, and we get the required result.   
\end{proof}

	


%

\bibliographystyle{named}
\bibliography{litt_std}

\begin{thebibliography}{}

\bibitem[\protect\citeauthoryear{Baltag and
  Smets}{2008}]{baltag2008qualitative}
Alexandru Baltag and Sonja Smets.
\newblock A qualitative theory of dynamic interactive belief revision.
\newblock In Giacomo Bonanno, Wiebe van~der Hoek, and Michael Wooldridge,
  editors, {\em Logic and the Foundations of Game and Decision Theory (LOFT7)},
  volume~3 of {\em Texts in Logic and Games}, pages 13--60. Amsterdam
  University Press, 2008.

\bibitem[\protect\citeauthoryear{Baltag \bgroup \em et al.\egroup
  }{1998}]{baltag1998logic}
Alexandru Baltag, Lawrence~S. Moss, and Slawomir Solecki.
\newblock The logic of public announcements and common knowledge and private
  suspicions.
\newblock In Itzhak Gilboa, editor, {\em Proceedings of the 7th Conference on
  Theoretical Aspects of Rationality and Knowledge ({TARK}-98)}, pages 43--56.
  Morgan Kaufmann, 1998.

\bibitem[\protect\citeauthoryear{Belardinelli and
  Bolander}{2023}]{belardinelli2023attention}
Gaia Belardinelli and Thomas Bolander.
\newblock {Attention! Dynamic Epistemic Logic Models of (In)Attentive Agents}.
\newblock In {\em Proceedings of the 2023 International Conference on
  Autonomous Agents and Multiagent Systems}, AAMAS '23, page 391–399,
  Richland, SC, 2023. International Foundation for Autonomous Agents and
  Multiagent Systems.

\bibitem[\protect\citeauthoryear{Belardinelli}{2023}]{belardinelli2024phdthesis}
Gaia Belardinelli.
\newblock {\em {Epistemic Logic Models of Awareness, Attention, and Implicit
  Reasoning}}.
\newblock PhD thesis, K{\o}benhavns Universitet, 2023.

\bibitem[\protect\citeauthoryear{Blackburn \bgroup \em et al.\egroup
  }{2001}]{blac.ea:moda}
P.~Blackburn, M.~de~Rijke, and Y.~Venema.
\newblock {\em Modal Logic}, volume~53 of {\em Cambridge Tracts in Theoretical
  Computer Science}.
\newblock Cambridge University Press, Cambridge, UK, 2001.

\bibitem[\protect\citeauthoryear{Bolander and
  Lequen}{2023}]{bolander2023parameterized}
Thomas Bolander and Arnaud Lequen.
\newblock Parameterized complexity of dynamic belief updates: A complete map.
\newblock {\em Journal of Logic and Computation}, 33(6):1270--1300, 2023.

\bibitem[\protect\citeauthoryear{Bolander \bgroup \em et al.\egroup
  }{2016}]{bolander2015announcements}
Thomas Bolander, Hans van Ditmarsch, Andreas Herzig, Emiliano Lorini, Pere
  Pardo, and Fran{\c{c}}ois Schwarzentruber.
\newblock {Announcements to Attentive Agents}.
\newblock {\em Journal of Logic, Language and Information}, 25:1--35, 2016.

\bibitem[\protect\citeauthoryear{Bolander \bgroup \em et al.\egroup
  }{2020}]{bolander2020del}
Thomas Bolander, Tristan Charrier, Sophie Pinchinat, and Francois
  Schwarzentruber.
\newblock {DEL}-based epistemic planning: Decidability and complexity.
\newblock {\em Artificial Intelligence}, 287:1--34, 2020.

\bibitem[\protect\citeauthoryear{Bolander \bgroup \em et al.\egroup
  }{2021}]{bolander2021del}
Thomas Bolander, Lasse Dissing, and Nicolai Herrmann.
\newblock Del-based epistemic planning for human-robot collaboration: Theory
  and implementation.
\newblock In {\em Proceedings of the 18th International Conference on
  Principles of Knowledge Representation and Reasoning (KR 2021)}, 2021.

\bibitem[\protect\citeauthoryear{Bolander}{2018}]{bolander2018seeing}
Thomas Bolander.
\newblock Seeing is believing: Formalising false-belief tasks in dynamic
  epistemic logic.
\newblock In {\em Outstanding Contributions to Logic}, number~12, pages
  207--236. Springer, 2018.

\bibitem[\protect\citeauthoryear{Boyd \bgroup \em et al.\egroup
  }{2011}]{boyd2011cultural}
Robert Boyd, Peter~J Richerson, and Joseph Henrich.
\newblock The cultural niche: Why social learning is essential for human
  adaptation.
\newblock {\em Proceedings of the National Academy of Sciences},
  108(supplement\_2):10918--10925, 2011.

\bibitem[\protect\citeauthoryear{Fagin and Halpern}{1988}]{fagin1988belief}
R.~Fagin and J.~Halpern.
\newblock Belief, awareness, and limited reasoning.
\newblock {\em Artificial Intelligence}, 34:39--76, 1988.

\bibitem[\protect\citeauthoryear{Fagin \bgroup \em et al.\egroup
  }{1995}]{fagin1995reasoning}
Ronald Fagin, Joseph~Y. Halpern, Yoram Moses, and Moshe~Y. Vardi.
\newblock {\em Reasoning About Knowledge}.
\newblock MIT Press, 1995.

\bibitem[\protect\citeauthoryear{Fritz and Lederman}{2015}]{fritz2016standard}
Peter Fritz and Harvey Lederman.
\newblock {Standard State Space Models of Unawareness (Extended Abstract)}.
\newblock {\em Electronic Proceedings in Theoretical Computer Science},
  215:141–158, 2015.

\bibitem[\protect\citeauthoryear{Johnson}{2024}]{johnson2024varieties}
Gabbrielle~M Johnson.
\newblock Varieties of bias.
\newblock {\em Philosophy Compass}, 19(7):e13011, 2024.

\bibitem[\protect\citeauthoryear{Kooi and Renne}{2011a}]{kooi2011arrow}
Barteld Kooi and Bryan Renne.
\newblock Arrow update logic.
\newblock {\em The Review of Symbolic Logic}, 4(4):536--559, 2011.

\bibitem[\protect\citeauthoryear{Kooi and Renne}{2011b}]{kooi2011generalized}
Barteld Kooi and Bryan Renne.
\newblock Generalized arrow update logic.
\newblock In {\em Proceedings of the 13th Conference on Theoretical Aspects of
  Rationality and Knowledge}, pages 205--211. ACM, 2011.

\bibitem[\protect\citeauthoryear{Li}{2023}]{li2023phdthesis}
Lei Li.
\newblock {\em Games, Boards and Play: A Logical Perspective}.
\newblock PhD thesis, University of Amsterdam, 2023.

\bibitem[\protect\citeauthoryear{Munton}{2023}]{munton2023prejudice}
Jessie Munton.
\newblock Prejudice as the misattribution of salience.
\newblock {\em Analytic Philosophy}, 64(1):1--19, 2023.

\bibitem[\protect\citeauthoryear{Rendell \bgroup \em et al.\egroup
  }{2010}]{rendell2010copy}
Luke Rendell, Robert Boyd, Daniel Cownden, Marquist Enquist, Kimmo Eriksson,
  Marc~W Feldman, Laurel Fogarty, Stefano Ghirlanda, Timothy Lillicrap, and
  Kevin~N Laland.
\newblock {Why copy others? Insights from the social learning strategies
  tournament}.
\newblock {\em Science}, 328(5975):208--213, 2010.

\bibitem[\protect\citeauthoryear{Simons and Chabris}{1999}]{simons1999gorillas}
Daniel~J. Simons and Christopher~F. Chabris.
\newblock Gorillas in our midst: Sustained inattentional blindness for dynamic
  events.
\newblock {\em Perception: Visual perception}, 28(9):1059--1074, 1999.

\bibitem[\protect\citeauthoryear{Smith and Archer}{2020}]{smith2020epistemic}
Leonie Smith and Alfred Archer.
\newblock Epistemic injustice and the attention economy.
\newblock {\em Ethical Theory and Moral Practice}, 23(5):777--795, 2020.

\bibitem[\protect\citeauthoryear{van Ditmarsch and
  Kooi}{2008}]{ditmarsch2008semantic}
Hans van Ditmarsch and Barteld Kooi.
\newblock Semantic results for ontic and epistemic change.
\newblock In Giacomo Bonanno, Wiebe van~der Hoek, and Michael Wooldridge,
  editors, {\em Logic and the Foundation of Game and Decision Theory (LOFT 7)},
  Texts in Logic and Games 3, pages 87--117. Amsterdam University Press, 2008.

\bibitem[\protect\citeauthoryear{van Ditmarsch \bgroup \em et al.\egroup
  }{2007}]{ditmarsch2007dynamic}
Hans van Ditmarsch, Wiebe van~der Hoek, and Barteld Kooi.
\newblock {\em Dynamic Epistemic Logic}.
\newblock Springer Publishing Company, 2007.

\bibitem[\protect\citeauthoryear{{van
  Ditmarsch}}{2023}]{ditmarsch2023announced}
Hans {van Ditmarsch}.
\newblock To be announced.
\newblock {\em Information and Computation}, 292:105026, 2023.

\bibitem[\protect\citeauthoryear{van Eijck \bgroup \em et al.\egroup
  }{2012}]{eijck2012action}
Jan van Eijck, Ji~Ruan, and Tomasz Sadzik.
\newblock Action emulation.
\newblock {\em Synthese}, 185:131--151, 2012.

\bibitem[\protect\citeauthoryear{Watzl}{2017}]{watzl2017structuring}
Sebastian Watzl.
\newblock {\em Structuring mind: The nature of attention and how it shapes
  consciousness}.
\newblock Oxford University Press, 2017.

\end{thebibliography}

\end{document}